\pdfoutput=1
\documentclass{article}
\usepackage{arxiv}  

\usepackage{amsmath,amssymb,amsthm}
\usepackage{graphicx}
\usepackage{booktabs}
\usepackage{algorithm}
\usepackage{algorithmic}
\usepackage{microtype}
\usepackage{xcolor}
\usepackage{tikz}
\usetikzlibrary{arrows.meta,positioning,patterns,calc}
\usepackage[font=small,labelfont=bf]{caption}
\usepackage{natbib}
\newcommand{\kwta}{$k$-WTA}
\usepackage[colorlinks=true,linkcolor=blue!60!black,citecolor=blue!60!black,urlcolor=blue!60!black]{hyperref}

\newtheorem{proposition}{Proposition}
\newtheorem{corollary}{Corollary}
\newtheorem{definition}{Definition}
\newtheorem{remark}{Remark}

\title{Replay in the Silent Degrees of Freedom: Continual Learning Without an Offline Phase}
\renewcommand{\shorttitle}{Replay in the Silent Degrees of Freedom}
\renewcommand{\undertitle}{Preprint}
\renewcommand{\headeright}{Preprint}

\author{%
  \begin{tabular}{@{}c@{\hspace{1.6em}}c@{\hspace{1.6em}}c@{}}
    Yanhai Zhang$^{1}$ & Jie Zhang$^{2}$ & Chi Xu$^{1,2}$ \\[1pt]
    \texttt{Zhang\_Yanhai@a-star.edu.sg} & \texttt{ZhangJ@ntu.edu.sg} & \texttt{Xu\_Chi@a-star.edu.sg}
  \end{tabular} \\[3pt]
  $^{1}$Agency for Science, Technology and Research (A\raisebox{0.12ex}{$\star$}STAR), Singapore \\
  $^{2}$College of Computing and Data Science, Nanyang Technological University, Singapore}
\date{\today}

\begin{document}
\maketitle

\begin{abstract}
Brains consolidate memories not only in sleep but also, as recent causal evidence shows,
through \emph{local sleep}: brief, use-dependent off-periods of individual cortical circuits
during wakefulness. Replay-based continual learning in artificial networks, by contrast,
almost always consolidates in a dedicated offline phase or by interleaving replayed samples
with the input stream. We ask whether a network trained by local, biologically constrained
learning rules can consolidate \emph{during inference}, with no offline phase at all. The
construction has three parts. An \emph{isolation} rule confines replay updates to hidden
synapses that are invisible to the current input under $k$-winner-take-all (\kwta) dynamics,
with optimiser state advanced only inside the mask. A \emph{refractory rotation} rule makes
units that have just fired sit out the next competition, which widens the set of consolidable
synapses. Two internal signals, a unit-level homeostatic pressure and a relative-novelty gate,
decide when replay bursts fire and when rotation runs. This inverts the usual direction of
non-interfering continual learning: rather than protecting past tasks while learning the
present one, the hidden computation on the current input is held invariant (exactly for silent and suppressed units when the mask matches the weights being updated, and for
all but $0.3\%$ of waking samples per update in the implemented form)
while past memories are written into the degrees of freedom the current batch leaves unused.
On class-incremental split-MNIST the system reaches $91.6\pm0.3\%$ with no offline phase, at
or above the best offline-night schedule on two held-out splits and consistently above
experience replay, ER-ACE, A-GEM and unmasked local replay, each re-tuned under the same
micro-batch schedule, at about twice the night's replay samples. Against DER++ the
comparison splits by protocol: with five epochs per task DER++ leads by $1.5$ points once it
runs under that schedule, and in a single pass over the stream the system leads it by $1.6$
while the night falls to $76.9\%$. Rotation carries most of the gain; isolation adds
the invariance guarantee and, at two-sample replay batches, avoids the severe failures that
unmasked replay shows in half the seeds. The advantage is largest at small episodic buffers
and gives way to the backpropagation references at large ones. On split CIFAR-10 the system
leads offline rehearsal and BP+ER but trails ER-ACE and DER++ by two to three points, and a
substrate decomposition on feature inputs attributes the gap to BP+ER to dense activation
rather than to credit assignment. The construction is not tied to the local learner: on a
backpropagation network with $k$-winner-take-all hidden layers under the same schedule,
refractory rotation adds half a point in five epochs and three in a single pass, and
isolation again costs nothing on top of it. Finally, a drive-referenced controller for the
rule's
synaptic decay replaces per-dataset tuning with a single ratio target and, with locally
learned skip synapses, holds a five-hidden-layer network at the two-layer level.
\end{abstract}

\section{Problem Formulation}
\label{sec:problem}

\paragraph{Setting.}
Let a stream of labelled examples $(x_t, y_t) \in \mathbb{R}^{d}\times\{1,\dots,C\}$ be presented
in $T$ contiguous tasks $\mathcal{T}_1,\dots,\mathcal{T}_T$, where task $\mathcal{T}_\tau$
contains only classes $\mathcal{C}_\tau \subset \{1,\dots,C\}$, the $\mathcal{C}_\tau$ are
disjoint, and no task identity is available at test time (class-incremental learning). A network
$f_\theta$ with a single shared $C$-way readout is trained on the stream and evaluated on all $C$
classes after the final task,
\begin{equation}
A_{\mathrm{seq}}(\theta) \;=\; \Pr_{(x,y)\sim\mathcal{D}_{\mathrm{test}}}\!\big[\arg\max_c f_\theta(x)_c = y\big],
\qquad
F \;=\; \frac{1}{T-1}\sum_{\tau=1}^{T-1}\big(a_{\tau,\tau} - a_{T,\tau}\big),
\end{equation}
where $a_{t,\tau}$ is accuracy on task $\tau$ after learning task $t$ and $F$ is forgetting.

\paragraph{Constraints.}
We restrict the learner to a set $\mathcal{B}$ of cortical constraints: no weight transport
(feedback synapses are learned, never copied), errors carried by a bounded burst-like channel,
Dale's law, sparse population codes (\kwta{} on wide layers), sparse distance-dependent
connectivity, and a single forward--feedback sweep per example (no iterative relaxation). An
episodic memory (``hippocampus'') of at most $K$ raw examples with reservoir writing
\citep{vitter1985} is permitted;
$K$ is a budget, not a free parameter.

\paragraph{Objective.}
Rather than a single accuracy number, we optimise a four-way trade-off
\begin{equation}
\max_\theta \;\big(A_{\mathrm{seq}},\, A_{\mathrm{iid}},\, -R,\, -E\big)
\quad\text{subject to}\quad \theta \in \mathcal{B},
\label{eq:pareto}
\end{equation}
where $A_{\mathrm{iid}}$ is accuracy of the identical learner on the i.i.d.\ shuffle of the same
data (the \emph{static axis}: a brain-like mechanism must not damage ordinary learning), $R$ is
the replay cost in samples (synaptic work is reported separately), and $E$ an arithmetic-energy proxy
(Section~\ref{sec:setup}). The central question of this paper is whether the offline
consolidation phase can be removed entirely, so that consolidation occurs only \emph{during}
inference on the wake stream, without losing $A_{\mathrm{seq}}$, and at what cost in
$A_{\mathrm{iid}}$ and $R$.

\section{Related Work}
\label{sec:related}

\paragraph{Sleep-inspired consolidation in ANNs.}
The Bazhenov line implements sleep as a global offline phase: networks are converted to spiking
dynamics and driven by noise under local Hebbian plasticity, recovering old tasks after each new
one \citep{tadros2022}, after several tasks at once \citep{bazhenov2026}, in spiking networks with
coarse alternation of training and sleep \citep{golden2022}, and in equilibrium-propagation
networks combined with small rehearsal buffers \citep{kubo2025}. Wake--sleep frameworks with NREM
and REM stages \citep{sorrenti2024,robinson2022} and engineering systems such as SIESTA
\citep{harun2023} likewise schedule consolidation offline, and SESLR \citep{lin2026} appends a
noise-enhanced sleep phase, in which a spiking classifier on a frozen extractor trains only
on its replay buffer, to correct the recency bias left by online learning. In all of these,
sleep is global, offline, and clocked. None confines consolidation to currently unused units
during ongoing inference, and the bias correction such a phase delivers is what our plastic
readout receives continuously from isolated replay.

\paragraph{Gating, subnetwork and history-dependent competition methods.}
Context-dependent gating (XdG) activates a fixed random subnetwork per task but requires task
identity at train and test time \citep{masse2018}; learned gates \citep{tilley2023}, dendritic
context vectors from input prototypes \citep{iyer2022}, task-free sparsification by feedforward
and top-down errors \citep{lassig2023}, and gradient-recovered functional masks \citep{mckee2026}
remove the explicit label but derive the gate from the \emph{input identity}. Use-history-dependent
suppression of competition, in itself, has classical precedents: refractory winner-take-all
dynamics in physiological models of competitive learning \citep{kaski1994}, explicit refractory
competitive learning in which a recent winner is temporarily barred from winning \citep{maeda1999},
and, recently, spiking continual learners whose firing thresholds rise with each unit's activation
history \citep{shen2024} or whose top-$k$ selection is randomised \citep{shen2026} to diversify
recruitment across tasks, the wake-side counterpart of our random-alternation control. Our
rotation differs in
\emph{purpose}, not in primitive: recently wake-active units are benched for one competition
specifically to enlarge a dynamically reallocated \emph{replay channel}. Rotation creates
plastic degrees of freedom in which consolidation can proceed concurrently with inference,
rather than allocating representational resources across tasks.

\paragraph{Null-space and parameter-isolation methods.}
A mature line of work constructs non-interfering updates by protecting \emph{previous} tasks:
projecting new-task gradients orthogonally to old-task gradient or activation subspaces
\citep{farajtabar2020,saha2021}, training in the null space of accumulated feature covariance
\citep{wangnscl2021}, combining $k$-winner sparsity and heterogeneous dropout with such
projections \citep{abbasi2022}, or freezing pruned-and-frozen subnetworks
\citep{golkar2019,mallya2018}. Our construction inverts the protected direction: the
\emph{ongoing wake computation} is held invariant while \emph{old-memory replay} is written into
degrees of freedom that the current batch's sparse support exposes as instantaneously
invisible. No bases, SVDs, or task masks are stored, and the ``null space'' changes with every
batch at no cost. \citet{mckee2026} note that optimiser state and decoupled weight decay can
break structural guarantees even at zero gradient, and \citet{bricken2023} analyse stale
momentum in sparse learners; we build the corresponding requirement, mask-confined optimiser
state, into the exactness construction. Among interleaved-replay methods under
backpropagation, DER++ \citep{buzzega2020} distils the logits stored with each buffered
sample, ER-ACE \citep{caccia2022} restricts the incoming loss to the classes present in the
batch, and A-GEM \citep{chaudhry2019} projects the incoming gradient so that it does not
increase the loss on a buffer batch; all three serve as references below. Replay
\emph{scheduling} has also been studied as a learned policy \citep{klasson2023}; our
contribution there is a specific unit-level use-pressure trigger and its empirical
dissociation from novelty triggering.

\paragraph{Complementary learning systems.}
CLS-ER maintains fast/slow EMA copies of the weights as semantic memories \citep{arani2022} and
DualNet trains a slow self-supervised learner beside a fast supervised one \citep{pham2021}; both
raise the value of each replayed sample but still replay at every step. Our contribution is
orthogonal: we change \emph{where} (isolated synapses), \emph{when} (homeostatic bursts) and at
what granularity (micro-batches) replay is applied.

\paragraph{Pattern separation.}
Cerebellum-, mushroom-body- and dentate-gyrus-style sparse expansions
\citep{caycogajic2017,caycogajic2019,litwin2017} inspire continual learners such as the FlyModel
\citep{shen2021}, sparse distributed memory \citep{bricken2023} and DG-gated mixtures
\citep{kapoor2026}. We tested this family as a front-end and found no benefit
(Section~\ref{sec:results}).

\paragraph{Biology of local sleep.}
Sleep-like off-periods occur locally in awake, sleep-deprived animals \citep{vyazovskiy2011}.
Causal evidence is recent: optogenetically inducing ON/OFF alternation in one cortical hemisphere
of awake mice locally discharges sleep pressure and weakens synaptic markers (GluA1, pS845),
bilateral induction during sleep deprivation rescues memory consolidation, and an equal
\emph{tonic} reduction of firing does not discharge the pressure \citep{driessen2026}. Our refractory rotation is the algorithmic counterpart of
this alternation requirement, and our soft-rotation control (a reduction of gain rather than
an off-period) the counterpart of tonic reduction.

\section{Experiment Setup}
\label{sec:setup}

\paragraph{Data and protocol.}
Sequential axis: split-MNIST (five tasks of two classes, $5$ epochs per task) and split CIFAR-10
(same protocol; $32{\times}32$ luminance so that distance-dependent connectivity retains its
geometry). Static axis: the identical learner and budgets on the i.i.d.\ shuffle. Waking batches
have $n_w=16$ samples; the episodic buffer holds $K{=}1000$ raw samples under reservoir
writing ($200$ and $5000$ on the buffer-size axis). Forgetting $F$ is the mean drop of each
earlier task's accuracy from its own end to
the end of training (per-task matrices in Appendix~\ref{app:val}); evaluation runs with the
suppression mask off (\kwta{} only), so test order and batch size cannot affect it.
Configurations were selected on the official test splits, which served as the development
set, and then frozen. Every main-text comparison re-trains them with a stratified tenth of
the training data removed from the stream and evaluates on that tenth, which took part in no
selection; a second tenth, drawn with a different seed and sharing $9\%$ of its samples with
the first, replicates the headline rows (Appendix~\ref{app:val}). Results not re-run under
this protocol are labelled development-phase. Numbers are mean $\pm$ s.d.; the component
table, the buffer-size axis and the CIFAR-10 comparison carry six seeds, the second split
and the remaining cells three unless marked otherwise. Runs are deterministic at a fixed
CPU thread count; because the
reduction order feeds \kwta's discontinuity, a different thread count is effectively a
different seed.
A third regime uses a frozen V1-like feature front-end for CIFAR-10: $256$ normalised random
$6{\times}6{\times}3$ patches sampled from the training images act as fixed filters (stride $2$,
rectification against the per-filter mean response, quadrant average pooling; 1024-D, linear
probe $\approx49\%$). Nothing in the front-end is learned; the frozen representation serves as
the model input.
Section~\ref{sec:controller} additionally uses split CIFAR-100 (ten tasks of ten classes,
same protocol and budgets, luminance and feature variants). With tenfold fewer supervision
events per class it probes the learning rule far outside the regime its constants were tuned
in, and its absolute numbers are at the floor for every method tested (chance $1\%$, BP+ER
$6$--$10\%$); we use it for orderings and mechanisms, not as a competitive benchmark.

\paragraph{Architecture.}
$d$--$512$--$256$--$C$ with \kwta{} sparsity $10\%$ on the hidden layers (the fraction
selected in Section~\ref{sec:dial}; appendix results at the earlier $5\%$ fraction are marked
as such), $30\%$ distance-dependent connectivity on hidden synapses, Dale's law with $80\%$
excitatory units, and the single-sweep learning rule of Section~\ref{sec:method}. The
offline-night reference uses its own preferred narrower core ($256$--$128$).

\paragraph{Baselines.}
(i) Backpropagation with experience replay (BP+ER) at equal buffer $K$ (included in every
table), and, on the same network, buffer and optimiser, DER++, ER-ACE and A-GEM, each at its
best width and loss on the development split (Section~\ref{sec:cost});
(ii) the offline \emph{night}: after each waking epoch, $20$ replay batches of $256$ samples at
$3\times$ learning rate with no concurrent input (the strongest schedule in our sweep);
(iii) interleaved ER for the local learner; (iv) no buffer; (v) to test the mechanism off
the local learner, the same schedule on a backprop MLP with \kwta{} hidden layers
(Section~\ref{sec:bpk}).

\paragraph{Metrics.}
Final accuracy, forgetting $F$, replay cost $R$ (batches and samples), the isolated-channel width
$\rho_\ell$ (fraction of replay-activated units in layer $\ell$ that are asleep for the current
input), code overlap (mean pairwise Jaccard of task-active unit sets), participation-ratio
dimensionality, and an idealised 45-nm FP32 arithmetic-energy proxy \citep{horowitz2014}:
$0.9$\,pJ per FP32 accumulate (one per synaptic event) for spiking inference vs.\ $4.6$\,pJ
per multiply--accumulate for dense rate inference, after rate-to-spike conversion with thresholds
calibrated on training images and $T_s$ integration steps \citep{rueckauer2017}. The proxy
counts arithmetic only, with no data movement, membrane updates or comparisons, so it is a
relative measure rather than a hardware figure (a fabricated neuromorphic chip reports
$\approx26$\,pJ per synaptic event \citep{merolla2014convention}).

\section{Methodology}
\label{sec:method}

\subsection{Cortical base learner}

Layer activities are $a^0 = x$ and, for $\ell = 1,\dots,L$,
\begin{equation}
z^\ell = W^\ell a^{\ell-1} + b^\ell, \qquad
a^\ell = \Phi_k\!\big(\sigma(z^\ell)\odot(1-s^\ell)\big),
\label{eq:forward}
\end{equation}
where $\sigma$ is a rectifier, $s^\ell \in \{0,1\}^{n_\ell}$ is the (possibly empty)
\emph{suppression mask} of Section~\ref{sec:rotation}, and $\Phi_k$ keeps the $k$ largest entries
of each row and zeroes the rest (\kwta); a per-unit gain exists in the implementation for
homeostatic scaling but is fixed at $1$ throughout. Layers $1,\dots,L{-}1$ are hidden; layer $L$
is a linear $C$-way readout whose pre-activation $z^L$ is the output, scored by squared error
against the one-hot target $y$. Errors are produced by one top-down sweep,
\begin{equation}
\varepsilon^L = y - z^L, \qquad
\varepsilon^\ell = \sigma'(z^\ell)\odot\Big((B^\ell)^{\!\top} \big[\kappa\tanh(\varepsilon^{\ell+1}/\kappa)
\odot \mathbf{1}(a^{\ell+1}>0)\big]\Big),
\label{eq:sweep}
\end{equation}
i.e.\ a bounded ($|\varepsilon^\ell_i|\le\kappa\sum_j|B^\ell_{ji}|$ per component), event-gated
burst signal carried by learned feedback weights
$B^\ell\in\mathbb{R}^{n_{\ell+1}\times n_\ell}$ (no transport), in the spirit of
predictive-coding and burst-dependent learners with local plasticity
\citep{whittington2017,payeur2021}. Forward and feedback weights follow the same local product with a shared decay
$\lambda$ (Kolen--Pollack alignment \citep{kolen1994,akrout2019}),
\begin{equation}
\Delta W^\ell \propto \varepsilon^\ell (a^{\ell-1})^{\!\top} - \lambda W^\ell,\qquad
\Delta B^{\ell-1} \propto \varepsilon^\ell (a^{\ell-1})^{\!\top} - \lambda B^{\ell-1},
\label{eq:kp}
\end{equation}
with sign-concordant feedback under Dale's law. The default optimiser is heavy-ball SGD
($\mu=0.9$, one per-synapse velocity, no second-moment state); the $\epsilon$-floored masked
Adam of the earlier $5\%$ phases appears as a control (Section~\ref{sec:direction}) and in the
appendices. Equations~\eqref{eq:forward}--\eqref{eq:kp} were fixed by earlier
ablation ladders and are not varied here, with one exception: Section~\ref{sec:controller}
replaces the constant $\lambda$ by an adaptive per-layer controller.

\subsection{Exactly isolated replay}
\label{sec:isolation}

Let $X=\{x_b\}_{b=1}^{m}$ be the current waking batch, $a^\ell_{i,b}$ the activity of unit
$i$ on sample $b$, and $\mathcal{A}^\ell(X) = \{i : \exists b,\ a^\ell_{i,b} \neq 0\}$ the
\emph{awake set}. During the same forward step, a replay batch drawn from the buffer is
applied through the synapse mask
\begin{equation}
M^\ell_{ij} \;=\; \mathbf{1}\big[\, j \notin \mathcal{A}^{\ell-1}(X) \;\lor\; i \notin \mathcal{A}^{\ell}(X) \,\big],
\label{eq:mask}
\end{equation}
i.e.\ a synapse may take the replay update iff its presynaptic unit is silent for the current
input or its postsynaptic unit is asleep. The mask applies to the \emph{entire} replay-induced
parameter change, not only to the data gradient: with heavy-ball velocity $v$ ($\mu=0.9$),
the local update direction $G=\varepsilon^\ell(a^{\ell-1})^{\!\top}$ of \eqref{eq:kp} and the
per-step decay coefficient $\lambda$, a replay step is
\begin{equation}
v\leftarrow M\odot(\mu v+G)+(1-M)\odot v,\qquad
W\leftarrow W+\eta\,M\odot v-\lambda\,M\odot W,
\label{eq:maskedstep}
\end{equation}
so the velocity is frozen (not zeroed) outside the mask, the decay is confined to it, and
biases follow the unit mask $\mathbf{1}[i\notin\mathcal{A}^\ell(X)]$; under Adam the first and
second moments are confined the same way. As \citet{mckee2026} caution and
\citet{bricken2023} analyse for sparse learners, unmasked optimiser state would otherwise
re-inject the isolated update into awake synapses and void the guarantees below; the
accuracy consequences of this choice are examined in Section~\ref{sec:ablation}. The
readout row and its bias remain plastic, since old classes must stay calibrated; the
guarantees below concern the hidden computation. Within one waking step the order is:
forward pass on $X$ under the current suppression, unmasked waking update, awake sets read
from that forward pass, masked replay update on the updated weights
(Algorithm~\ref{alg:step}, Appendix~\ref{app:alg}). The replay micro-batch is inferred with
the suppression removed, so a refractory unit may fire on replay and learn from it, and a
burst of several micro-batches reuses the mask of its waking batch.

\begin{proposition}[Pre-silent updates are exactly invisible]
\label{prop:pre}
Fix $X$ and let $\widehat W^\ell = W^\ell + \Delta^\ell$ with $\Delta^\ell_{ij} = 0$
whenever $j \in \mathcal{A}^{\ell-1}(X)$. Then every hidden activity on every sample of $X$ is
unchanged for any magnitude of $\Delta$; if the readout row satisfies the same condition, or
is held fixed, the network output is unchanged as well.
\end{proposition}
\begin{proof}
By induction on $\ell$, for every sample $b$. Assume $a^{\ell-1}_{\cdot,b}$ is unchanged
(true for $\ell=1$ since $a^0_{\cdot,b}=x_b$). For any unit $i$,
$\widehat z^\ell_{i,b} - z^\ell_{i,b} = \sum_j \Delta^\ell_{ij}\, a^{\ell-1}_{j,b}$, and each
term vanishes: if $j\in\mathcal{A}^{\ell-1}(X)$ then $\Delta^\ell_{ij}=0$; otherwise
$a^{\ell-1}_{j,b}=0$ for every $b$. Hence $z^\ell_{\cdot,b}$, and therefore $a^\ell_{\cdot,b}$
through \eqref{eq:forward}, is unchanged; the induction closes at the last hidden layer, and
at the output layer whenever the readout row satisfies the same condition.
\end{proof}

\begin{proposition}[Post-asleep updates are invisible under a margin condition]
\label{prop:post}
Let $\Delta^\ell$ and $\Delta b^\ell$ additionally move synapses and biases of units
$i\notin\mathcal{A}^\ell(X)$ with active presynaptic partners, and write
$\delta_{i,b} = \sum_j \Delta^\ell_{ij} a^{\ell-1}_{j,b} + \Delta b^\ell_i$. If for every such
unit and every sample $\sigma(z^\ell_{i,b} + \delta_{i,b}) < \tau^\ell_{k,b}$, where
$\tau^\ell_{k,b}$ is the $k$-th largest value entering $\Phi_k$ on sample $b$ (strict, so a tie
cannot promote the unit; when fewer than $k$ entries are positive, $\tau^\ell_{k,b}=0$ and the
condition reads $\sigma(\cdot)=0$), the hidden activities on $X$ are unchanged.
\end{proposition}
\begin{proof}
A unit outside $\mathcal{A}^\ell(X)$ contributes $a^\ell_{i,b}=0$ on every sample. After the
update its pre-activation on sample $b$ is $z^\ell_{i,b}+\delta_{i,b}$; under the stated margin
it is still excluded by $\Phi_k$ (or rectified to zero), so $a^\ell_{i,b}=0$ still. The sample's
original winners keep their pre-activations, because their synapses from active presynaptic
units lie outside $\Delta^\ell$ by the mask and their silent-presynaptic terms vanish as in
Proposition~\ref{prop:pre}; so the same $k$ units win with the same values, every other
unit's input is untouched, and the layer's activities are identical. When fewer than $k$ units
are positive on sample $b$, $\tau^\ell_{k,b}=0$ and the condition reads
$\sigma(z^\ell_{i,b}+\delta_{i,b})=0$: a zero-valued entry carries no activity whether or not
$\Phi_k$ selects it.
\end{proof}

\begin{corollary}[Suppressed units are exact for any magnitude]
\label{cor:refr}
If $i$ is suppressed by the rotation mask ($s^\ell_i=1$ in \eqref{eq:forward}), then
$a^\ell_{i,b}=0$ for every sample and Proposition~\ref{prop:post} holds with no margin condition.
\end{corollary}
The propositions take the mask and the weights it protects from the same state. In the
implemented order the awake sets precede the waking update, so they hold exactly for the
pre-update weights; the one-step lag is part of the measured drift below and accounts for
$0.02$ of its $0.32$ hidden-code points.

\begin{remark}[How exact, measured]
The default system keeps the readout plastic and does not enforce the margin of
Proposition~\ref{prop:post}, so exactness beyond the proven channels is an empirical
question about the training-time computation under the suppression mask (evaluation runs
with the mask off). We measured it on the default configuration under the reported protocol
by re-inferring every waking batch, under the suppression it was inferred with, after each of
its $16{,}885$ isolated replay updates. The update altered the top hidden code of $0.33\%$ of
the waking samples per step (three seeds, $0.31$--$0.35$) and changed a waking prediction for
$0.32\%$, almost all of it through the plastic readout: with the readout isolated as well, the
prediction rate is $0.001\%$ and the hidden-code rate $0.21\%$ (one seed). The margin
condition was violated for $\sim\!10^{-4}$ of the asleep units per step, and the mean maximal
logit change per update was $0.02$ ($0.001$ with the readout isolated) on the one-hot scale.
On raw CIFAR-10 (one seed, development run) the hidden-code rate is $0.46\%$ while the
prediction rate rises to $3.6\%$, because low-margin predictions flip more easily through the
plastic readout. The rates are taken under the implemented order of Algorithm~\ref{alg:step}
(Appendix~\ref{app:alg}), in which the awake sets come from the activities before the waking
update; masking on awake sets re-inferred after that update instead lowers the hidden-code
rate only from $0.32\%$ to $0.30\%$ and leaves the prediction rate, the margin-violation rate
and the accuracy unchanged (one seed, development run), so the residual drift comes from the
unenforced margin rather than from the ordering. These rates concern the hidden code of the
\emph{current} batch; what replay does to later inputs is measured by the ablation of
Section~\ref{sec:ablation}. The isolation is thus exact for silent-presynaptic and suppressed
units and near-exact elsewhere, and Corollary~\ref{cor:refr} is why the refractory rotation
makes it robust rather than merely approximate.
\end{remark}

\subsection{Refractory rotation and its gates}
\label{sec:rotation}

\begin{definition}[Refractory rotation]
After processing waking batch $X_t$, set $s^\ell_i(t{+}1) = \mathbf{1}[i \in \mathcal{A}^\ell(X_t)]$:
every unit that fired sits out the next competition and is therefore asleep, and
consolidable, for batch $X_{t+1}$.
\end{definition}

Rotation widens the isolated channel
$\rho_\ell=|\tilde{\mathcal{A}}^\ell\setminus\mathcal{A}^\ell(X)|/|\tilde{\mathcal{A}}^\ell|$, the fraction of
the units $\tilde{\mathcal{A}}^\ell$ activated by the replay micro-batch that are asleep for the current
input, from $0.2$--$0.3$ (natural silence) to $\approx0.53$, at the price of running inference on the second-best coalition. Two internal
signals control the system (Figure~\ref{fig:mech}):

\paragraph{Homeostatic pressure (when to replay).} Each unit integrates its own use,
$S_i \leftarrow S_i + \bar a_i(X_t)$ with $\bar a_i(X_t)$ the fraction of the batch's samples on
which unit $i$ fired, and a replay \emph{burst} of $n_b$ consecutive micro-batches (one mask for
all of them) is triggered when the mean pressure over the currently asleep units of all hidden
layers exceeds $\theta$; consolidation discharges the pressure of the units that received it,
$S_i \leftarrow (1-\delta\,\mathbf{1}[i\ \mathrm{asleep}])\,S_i$ with $\delta=1$ throughout.
This is a unit-level analogue inspired by the homeostatic Process~S of sleep regulation
\citep{borbely1982}.

\paragraph{Relative novelty (when to rotate).} With fast and slow EMAs of the batch-mean
squared output error $e_t$,
$\bar e_f(t) = (1{-}\alpha_f)\bar e_f + \alpha_f e_t$, $\alpha_f{=}0.1$, and
$\bar e_s$, $\alpha_s{=}0.002$, rotation runs only while
\begin{equation}
\bar e_f(t) \;\ge\; \beta\, \bar e_s(t), \qquad \beta \approx 1.1\text{--}1.5,
\label{eq:nov}
\end{equation}
i.e.\ while the stream is \emph{new relative to the learner's own recent history} (loosely
inspired by cholinergic novelty signalling \citep{hasselmo2006}). Monitoring streaming error
with exponentially weighted averages is standard in concept-drift detection \citep{ross2012},
and related loss-ratio triggers appear in concurrent continual-learning work. The detector
itself is therefore not new; what is new is what it gates (the rotation, and hence the width
of the replay channel) and the unmastered clause below. A single absolute threshold cannot
serve across the regimes considered here: the converged error of a ten-class i.i.d.\ stream
exceeds any level a two-class task dips under, and we observe exactly this failure (gate open
$100\%$ of the time on static MNIST, $95\%$ on CIFAR, for any threshold that behaves on
split-MNIST).
Because rotation is \emph{beneficial} on streams the learner never masters
(Section~\ref{sec:results}), the final gate adds an ``unmastered'' clause: with $e_0$ the mean
error over the first $20$ batches (a chance-level estimate), rotation runs while
\begin{equation}
\bar e_f(t) \;\ge\; \beta\,\bar e_s(t)
\qquad\text{or}\qquad
\bar e_f(t) \;\ge\; \gamma\, e_0,
\qquad \gamma\approx0.5 .
\label{eq:prog}
\end{equation}
The first clause opens the gate at task switches, the second keeps it open on any stream whose
error remains a substantial fraction of chance; both are scale-free.

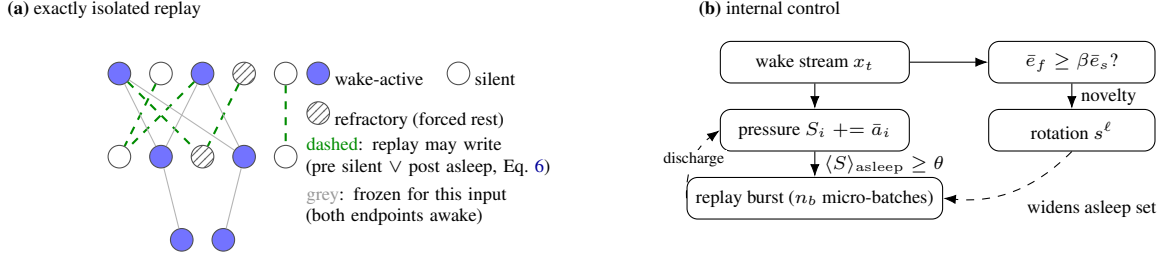
\begin{figure}[t]
\centering
\begin{tikzpicture}[
  font=\scriptsize, >=Latex,
  unit/.style={circle, draw=black!70, minimum size=8.5pt, inner sep=0pt},
  awake/.style={unit, fill=blue!55},
  asleep/.style={unit, fill=white},
  refr/.style={unit, pattern=north east lines, pattern color=black!60},
  lay/.style={},
]
\node at (-0.2,3.05) {\textbf{(a)} exactly isolated replay};
\foreach \i/\st in {0/awake,1/asleep,2/awake,3/refr,4/asleep}
  \node[\st] (x\i) at (0.55*\i, 2.2) {};
\foreach \i/\st in {0/asleep,1/awake,2/refr,3/awake,4/asleep}
  \node[\st] (h\i) at (0.55*\i, 1.1) {};
\foreach \i/\st in {1/awake,3/awake} \node at (0.55*\i, 0.0) {};
\node[awake] (o0) at (0.825, 0.0) {};
\node[awake] (o1) at (1.375, 0.0) {};
\draw[green!55!black, dashed, thick] (x1) -- (h0);
\draw[green!55!black, dashed, thick] (x3) -- (h2);
\draw[green!55!black, dashed, thick] (x4) -- (h4);
\draw[green!55!black, dashed, thick] (x2) -- (h0);
\draw[green!55!black, dashed, thick] (x0) -- (h2);
\draw[black!30] (x0) -- (h1); \draw[black!30] (x2) -- (h1);
\draw[black!30] (x0) -- (h3); \draw[black!30] (x2) -- (h3);
\draw[black!30] (h1) -- (o0); \draw[black!30] (h3) -- (o1);
\node[align=left, anchor=west] at (2.35, 2.2)
  {\tikz\node[awake]{}; wake-active \quad \tikz\node[asleep]{}; silent};
\node[align=left, anchor=west] at (2.35, 1.65)
  {\tikz\node[refr]{}; refractory (forced rest)};
\node[align=left, anchor=west, text width=4.3cm] at (2.35, 1.1)
  {\textcolor{green!55!black}{dashed}: replay may write\\ (pre silent $\lor$ post asleep, Eq.~\ref{eq:mask})};
\node[align=left, anchor=west, text width=4.3cm] at (2.35, 0.45)
  {\textcolor{black!40}{grey}: frozen for this input\\ (both endpoints awake)};
\begin{scope}[xshift=8.0cm]
\node at (0.6,3.05) {\textbf{(b)} internal control};
\node[draw, rounded corners, minimum width=2.5cm, minimum height=0.55cm] (wake) at (1.2, 2.35) {wake stream $x_t$};
\node[draw, rounded corners, minimum width=2.5cm, minimum height=0.55cm] (press) at (1.2, 1.45) {pressure $S_i \mathrel{+}= \bar a_i$};
\node[draw, rounded corners, minimum width=2.5cm, minimum height=0.55cm] (burst) at (1.2, 0.55) {replay burst ($n_b$ micro-batches)};
\node[draw, rounded corners, minimum width=2.2cm, minimum height=0.55cm] (nov) at (4.6, 2.35) {$\bar e_f \ge \beta \bar e_s$?};
\node[draw, rounded corners, minimum width=2.2cm, minimum height=0.55cm] (rot) at (4.6, 1.45) {rotation $s^\ell$};
\draw[->] (wake) -- (press);
\draw[->] (press) -- node[right]{$\langle S\rangle_{\mathrm{asleep}} \ge \theta$} (burst);
\draw[->, dashed] (burst.west) to[bend left=28] node[fill=white, inner sep=1pt]{\tiny discharge} (press.west);
\draw[->] (wake.east) -- (nov.west);
\draw[->] (nov) -- node[right]{novelty} (rot);
\draw[->, dashed] (rot.south) to[bend left=25] node[below right]{\ widens asleep set} (burst.east);
\end{scope}
\end{tikzpicture}
\caption{\textbf{Mechanism.} (a) During each waking batch, replay from the episodic buffer is
written only into synapses whose update is invisible to the current input, exactly so for
silent presynaptic and suppressed units (Propositions~\ref{prop:pre}--\ref{prop:post});
refractory rotation forces recently used units
into the consolidable set (Corollary~\ref{cor:refr}). (b) Two internal signals: unit-level
homeostatic pressure triggers replay bursts, and a relative-novelty gate decides when rotation
runs. The default system runs rotation always on and replays one micro-batch at every
waking batch; the pressure trigger and the novelty gate are the schedulers studied under
sparse replay budgets and lower activity fractions.}
\label{fig:mech}
\end{figure}

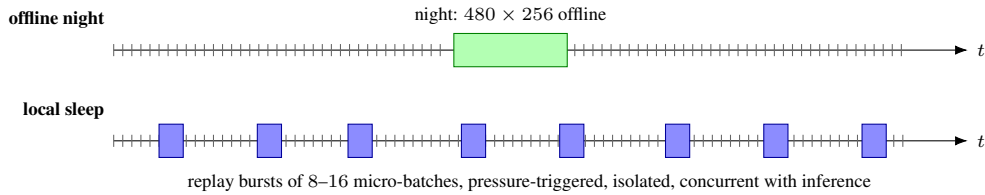
\begin{figure}[t]
\centering
\begin{tikzpicture}[font=\scriptsize, >=Latex]
\draw[->] (0,1.5) -- (11.3,1.5) node[right]{$t$};
\node[left] at (0,1.9) {\textbf{offline night}};
\foreach \x in {0,0.12,...,4.4} \draw[black!60] (\x,1.42) -- (\x,1.58);
\draw[fill=green!30, draw=green!50!black] (4.5,1.28) rectangle (6.0,1.72);
\node at (5.25,1.95) {night: $480 \times 256$ offline};
\foreach \x in {6.1,6.22,...,10.5} \draw[black!60] (\x,1.42) -- (\x,1.58);
\draw[->] (0,0.3) -- (11.3,0.3) node[right]{$t$};
\node[left] at (0,0.7) {\textbf{local sleep}};
\foreach \x in {0,0.12,...,10.5} \draw[black!60] (\x,0.22) -- (\x,0.38);
\foreach \x in {0.6,1.9,3.1,4.6,5.9,7.3,8.6,9.9} \draw[fill=blue!45, draw=blue!60!black] (\x,0.08) rectangle (\x+0.32,0.52);
\node at (5.5,-0.25) {replay bursts of $8$--$16$ micro-batches, pressure-triggered, isolated, concurrent with inference};
\end{tikzpicture}
\caption{\textbf{Consolidation schedules.} Ticks are waking batches. The offline night pauses the
stream; local sleep consolidates during it, in replay bursts confined to asleep units.}
\label{fig:timeline}
\end{figure}

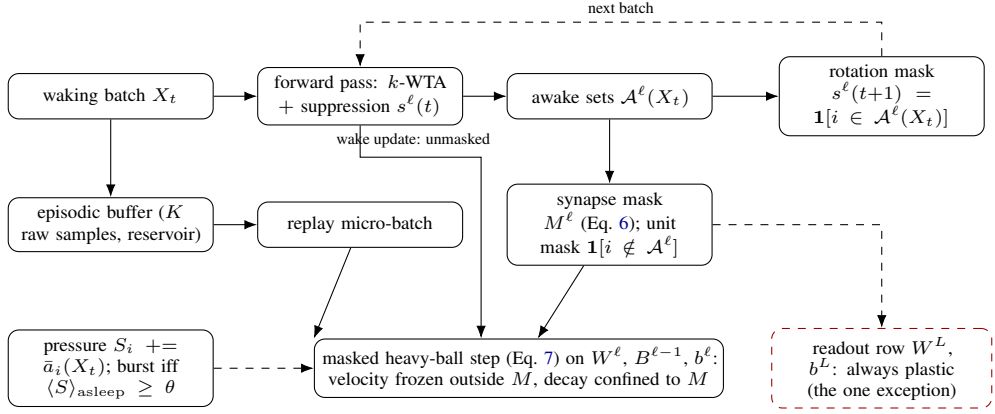
\begin{figure}[h]
\centering
\begin{tikzpicture}[font=\scriptsize, >=Latex,
  box/.style={draw, rounded corners, inner sep=3pt, align=center, minimum height=6mm, text width=2.5cm},
  wide/.style={box, text width=5.2cm},
  exc/.style={box, dashed, draw=red!60!black, text width=2.7cm}]
\node[box] (wake)  at (0,0)     {waking batch $X_t$};
\node[box] (fwd)   at (3.3,0)   {forward pass: \kwta{} $+$ suppression $s^\ell(t)$};
\node[box] (awake) at (6.6,0)   {awake sets $\mathcal{A}^\ell(X_t)$};
\node[box] (rot)   at (10.2,0)  {rotation mask $s^\ell(t{+}1)=\mathbf{1}[i\in\mathcal{A}^\ell(X_t)]$};
\node[box] (buf)   at (0,-1.7)  {episodic buffer ($K$ raw samples, reservoir)};
\node[box] (rep)   at (3.3,-1.7){replay micro-batch};
\node[box] (mask)  at (6.6,-1.7){synapse mask $M^\ell$ (Eq.~\ref{eq:mask}); unit mask $\mathbf{1}[i\notin\mathcal{A}^\ell]$};
\node[box] (press) at (0,-3.6)  {pressure $S_i \mathrel{+}= \bar a_i(X_t)$; burst iff $\langle S\rangle_{\mathrm{asleep}}\ge\theta$};
\node[wide](step)  at (5.4,-3.6){masked heavy-ball step (Eq.~\ref{eq:maskedstep}) on $W^\ell$, $B^{\ell-1}$, $b^\ell$: velocity frozen outside $M$, decay confined to $M$};
\node[exc] (ro)    at (10.2,-3.6){readout row $W^L$, $b^L$: always plastic (the one exception)};
\draw[->] (wake) -- (fwd); \draw[->] (fwd) -- (awake); \draw[->] (awake) -- (rot);
\draw[->] (awake) -- (mask); \draw[->] (mask) -- (step);
\draw[->] (buf) -- (rep); \draw[->] (rep) -- (step.north west);
\draw[->] (wake) -- (buf);
\draw[->] (fwd.south) -- ++(0,-0.42) -| node[pos=0.22, above, font=\tiny]{wake update: unmasked} ([xshift=-0.5cm]step.north);
\draw[->, dashed] (press) -- (step);
\draw[->, dashed] (rot.north) |- ++(0,0.45) -| node[pos=0.25, above, font=\tiny]{next batch} (fwd.north);
\draw[->, dashed] (mask.east) -| (ro.north);
\end{tikzpicture}
\caption{\textbf{Training data flow.} Each waking batch is inferred under the current
suppression mask; its awake sets define the synapse and unit masks for the replay micro-batch
applied in the same step and the rotation mask for the next batch. The masked step confines
the velocity and the decay to the mask; the readout row is the one deliberately plastic
exception. Unit-level pressure decides when replay bursts run.}
\label{fig:dataflow}
\end{figure}

\subsection{Cost accounting}
\label{sec:cost}

\paragraph{Updates and wall-clock.} Samples are not the only cost. The headline schedule
makes $16{,}885$ masked updates against the night's $480$ ($35\times$; counts from the
reported runs, whose stream is nine tenths of the training set), and its training wall-clock
on one CPU thread is $8.8$ min against $1.4$ for the night, $1.0$ with no replay and $8.6$
for unmasked interleaved replay (one run each, same protocol and thread count). The
per-batch replay step, not the sample count, is the online price of consolidating during
the stream, in an unoptimised implementation whose mask construction is dense.
\paragraph{Baseline tuning and paired comparison.} The offline-night reference was swept over
replay batch count ($20$, $40$ per night), replay batch size ($16$, $64$, $256$), learning-rate
gain ($3\times$), core width ($256$--$128$ and $512$--$256$), buffer size ($200$, $1000$,
$5000$) and timing (clocked after each epoch, novelty-timed, surprise-timed), and its best
sequential cell is reported. BP+ER was swept over learning rate ($3{\cdot}10^{-4}$,
$10^{-3}$, $3{\cdot}10^{-3}$), width and buffer size, and the local learner over $\eta$
(Appendix~\ref{app:eta}), replay batch size, cadence and the gates of
Table~\ref{tab:ablation}. DER++, ER-ACE and A-GEM share BP+ER's network, Adam optimiser,
batch of $256$ waking samples with a replay batch of the same size, and reservoir buffer;
each was swept over width ($256$--$128$, $512$--$256$) and, where the method allows it, over
the loss on the incoming batch (squared error on the one-hot target, as for BP+ER, or
cross-entropy), DER++ additionally over $\alpha\in\{0.01,0.03,0.1,0.3,1\}$ and
$\beta\in\{0.5,1\}$, where $\alpha$ multiplies the per-sample sum of squared logit
differences (ten times the mean-reduced form of the published code); ER-ACE is defined with
cross-entropy. The selected cells on split-MNIST are DER++ with cross-entropy, $\alpha=0.03$
($0.3$ in the mean-reduced form), $\beta=1$ at width $512$--$256$ (development split $92.2$,
with every $\alpha$ from $0.01$ to $1$ within $0.2$ of it; its squared-error variants reach
$89.9$ at best), ER-ACE at $256$--$128$ ($90.4$) and A-GEM with squared error at $512$--$256$
($73.6$; with cross-entropy A-GEM is unstable, $37$--$43$). On raw CIFAR-10 the same sweep
selects DER++ with cross-entropy, $\alpha=1$, $\beta=1$ at $512$--$256$ ($31.1$), ER-ACE at
$512$--$256$ ($31.7$) and A-GEM at $512$--$256$ ($16.5$, the no-replay level). Under the
micro-batch schedule of Section~\ref{sec:bpk} the three were re-swept over
$\{3{\cdot}10^{-6},10^{-5},3{\cdot}10^{-5},10^{-4},3{\cdot}10^{-4},10^{-3}\}$, keeping the
$\alpha$ and $\beta$ selected at batch $256$: DER++ and ER select $3{\cdot}10^{-5}$
(development split $93.4$ and $90.9$) and ER-ACE $3{\cdot}10^{-4}$ ($90.3$), each an
interior optimum of the grid. The backprop
\kwta{} learner of Section~\ref{sec:bpk} was swept over Adam learning rates
$\{10^{-5},3{\cdot}10^{-5},10^{-4},3{\cdot}10^{-4},10^{-3},3{\cdot}10^{-3}\}$ under the local
schedule; interleaved replay selects $3{\cdot}10^{-5}$ (development split $92.4$) and the
three other variants $10^{-4}$ ($92.9$ and $92.9$ for the two rotation variants, $84.3$ for
isolation alone). Selection
everywhere was by sequential accuracy at $K{=}1000$ on the official test split (the
development set), with the static axis reported alongside; the comparisons reported are on
the held-out split (Appendix~\ref{app:val}).
Pairing seeds by index on that split, the local system's margin over the narrow-substrate
night is $+3.1$ points (bootstrap 95\% CI $[+1.4,+5.4]$, six seed pairs), over the
same-substrate night $+2.6$ ($[+0.2,+5.2]$, six pairs; the night's seed variance makes this
one fragile, and on the second split it shrinks to $+0.3$), over unmasked interleaved ER
$+6.3$ ($[+3.8,+9.2]$), over BP+ER $+2.8$ ($[+2.5,+3.2]$), over DER++ $+0.1$
($[-0.3,+0.4]$), over ER-ACE $+1.5$ ($[+1.1,+1.9]$), over A-GEM $+23.2$ ($[+18.3,+28.6]$),
over the silent mask $+3.6$ ($[+3.1,+4.1]$) and over rotation alone $+0.1$
($[-0.3,+0.5]$).
Replay cost is measured in samples ($R_s$ = batches $\times$ batch size) and synaptic work.
The night costs $R_s = 480\times256 \approx 1.2\times10^5$ samples. The headline
configuration (one $16$-sample micro-batch per waking batch, $91.6\pm0.3$) costs
$16{,}885\times16\approx2.7\times10^5$ samples ($2.2\times$ the night). Two schedules meet
the night's budget within $1.1\times$: eight-sample micro-batches at every waking batch
($1.35\times10^5$, $91.1\pm0.4$) and $16$-sample micro-batches at every second one ($8{,}442$
updates, $1.35\times10^5$, $91.2\pm0.2$ sequential, $93.1\pm0.2$ static);
pressure-triggered bursts of $16$ at a third of the events reach $88.6\pm1.6$. Accuracy
therefore depends on the frequency of small isolated updates rather than on the number of
replayed samples.

\section{Results}
\label{sec:results}

\subsection{Local sleep replaces the night}
\label{sec:r1}

On split-MNIST at $K{=}1000$, isolated replay with refractory rotation, one replay
micro-batch beside every waking batch, reaches $\mathbf{91.6\pm0.3\%}$ (six seeds) with no
offline phase. This is above the best offline night ($88.5\pm2.8\%$ on the night's preferred
narrow substrate; $89.0\pm3.5$, highly variable across seeds, on ours; paired margins $+3.1$
$[+1.4,+5.4]$ and $+2.6$ $[+0.2,+5.2]$), above unmasked interleaved ER at the same budget
($87.5\pm0.8$ under Adam, $85.4\pm3.5$ under SGD), above BP+ER ($88.8\pm0.3$; $+2.8$
$[+2.5,+3.2]$), ER-ACE ($90.2\pm0.4$; $+1.5$ $[+1.1,+1.9]$) and A-GEM ($68.5\pm7.0$, whose
gradient projection protects the past at the expense of the present), and level with DER++
($91.5\pm0.5$; $+0.1$ $[-0.3,+0.4]$), the one reference it does not beat. Those three
references run under their own batch-$256$ protocol; Section~\ref{sec:bpk} re-tunes and
re-runs them under ours, where the system keeps its lead over ER and ER-ACE and gives up
$1.5$ points to DER++. In a single pass
over the stream (one epoch per task, six seeds) the ordering sharpens: $91.8\pm0.3$ against
$90.1\pm0.7$ for DER++, $88.6\pm0.2$ for BP+ER, $88.5\pm0.3$ for ER-ACE, $83.0\pm14.9$ for
unmasked ER and $76.9\pm3.4$ for the offline night, for which one night per task is too few
(Table~\ref{tab:ablation}); under the micro-batch schedule the same ordering holds, with
DER++ at $90.2\pm0.4$ and ER-ACE at $90.5\pm0.6$. Halving the replay cadence with $16$-sample batches costs $0.4$
($91.2\pm0.2$). Adding a night on top of continuous local sleep adds $0.8$ ($92.4\pm0.3$ vs.\
$91.6\pm0.3$): once consolidation runs during wake, the offline phase is nearly redundant at
this buffer size.

\paragraph{Buffer size.} Figure~\ref{fig:ksweep} varies the buffer. At $K{=}200$ the system
leads every reference ($83.0\pm1.4$ against $82.7\pm2.9$ for the narrow night,
$77.1\pm0.9$ for DER++ and $68.5\pm2.4$ for BP+ER); at $K{=}1000$ it ties DER++; at
$K{=}5000$ the backprop references pull ahead ($95.7\pm0.2$ for DER++ and $95.2\pm0.1$ for
BP+ER against $93.8\pm0.2$), while the offline night on the local learner saturates near
$88$ ($88.5\pm2.8$ at $K{=}1000$, $88.0\pm3.3$ at $5000$). The mechanism's advantage is
therefore largest where the episodic store is smallest, and at large buffers the ceiling is
the local learner's own i.i.d.\ accuracy ($94.1$), not the mechanism. The development-phase
study of tiny buffers at the $5\%$ fraction (Appendix~\ref{app:k200}) had found the narrow
night ahead there; on the held-out split at the $10\%$ fraction the two are at parity.

\begin{figure}[t]
\centering
\includegraphics[width=.5\linewidth]{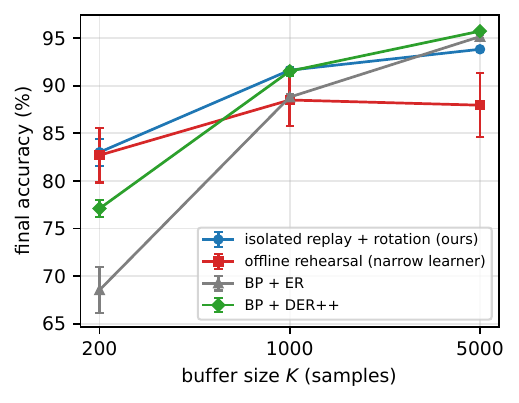}
\caption{\textbf{Buffer size.} Split-MNIST sequential accuracy against the buffer size $K$
for the system and the replay references (held-out split, six seeds).}
\label{fig:ksweep}
\end{figure}

\paragraph{A second split.} On a second, independent held-out tenth the ordering holds
(rotation $91.8\pm0.4$, DER++ $91.3\pm0.3$, ER-ACE $90.4\pm0.2$, BP+ER $88.5\pm0.7$, unmasked
ER $80.0\pm14.1$, A-GEM $70.4\pm4.1$, the narrow night $89.3\pm0.6$, the silent mask
$88.6\pm1.3$, rotation alone $91.8\pm0.6$), with one exception: the same-substrate night,
bimodal on the first split, reaches $91.5\pm0.5$ on the second and trails by only $0.3$
there (Appendix~\ref{app:val}).

\subsection{What rotation and isolation each buy}
Figure~\ref{fig:batch} (left) varies the replay batch size. The full system is flat from
$256$ down to $8$ samples ($91.6\pm0.5\to91.1\pm0.4$) and loses $1.9$ at $4$ ($89.7\pm0.7$).
At batch $16$, unmasked replay without rotation reaches $85.4\pm3.5$, isolation without
rotation $88.0\pm0.7$, and the offline night falls from $89.0$ to $85.8\pm3.0$ when its
batches shrink to $16$. The fourth cell of the square, rotation with unmasked replay, reaches
$91.5\pm0.3$ at batch $16$, $90.9\pm0.1$ at $8$ and $89.3\pm0.7$ at $4$. The two factors are
therefore strongly sub-additive: rotation is the larger single factor ($+6.1$ alone, $+3.6$
on top of isolation), and isolation's seed-paired margin is $0.1$ $[-0.3,+0.5]$ at batch
$16$, $0.3$ $[-0.2,+0.6]$ at $8$ and $0.3$ $[-0.5,+1.2]$ at $4$, within seed noise at every
size, with $0.4$ on the static axis ($94.1$ vs.\ $93.7$). At two-sample batches both systems
lose heavily, but differently. Isolated replay degrades gracefully ($79.0\pm3.2$, six seeds
within $75$--$84$), whereas unmasked replay bifurcates: three seeds finish at $82$--$85$, one
at $58$ and two at chance, i.e.\ severe failures below $60\%$ in half the seeds
(Appendix~\ref{app:isomargin}). In the seeds and configurations tested, isolation's accuracy
value is thus the absence of those failures rather than a margin. What it buys at every
batch size is the guarantee that, for the current input under its suppression mask, replay
leaves the hidden activities unchanged, exactly on the proven channels and within the margin
elsewhere. Without it every micro-batch perturbs the live coalition, and the outcome rests on
the rotation's tolerance of that perturbation, which runs out at batch $2$.

\begin{figure}[t]
\centering
\includegraphics[width=.48\linewidth]{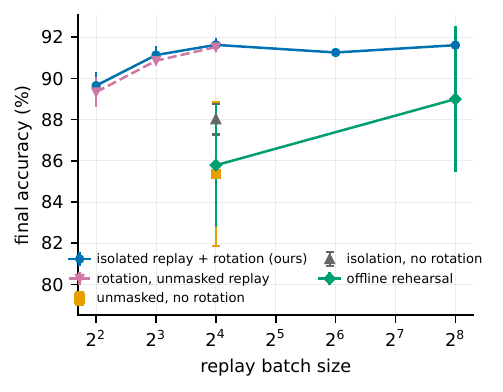}\hfill
\includegraphics[width=.48\linewidth]{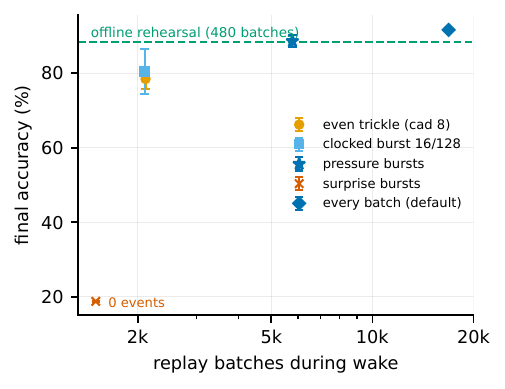}
\caption{\textbf{Left (isolation):} final split-MNIST accuracy vs.\ replay batch size
(default substrate, heavy-ball SGD); the number of updates is fixed within each series (one
per waking batch; $20$ per epoch for the night, whose samples therefore fall with its batch).
The full system is flat down to $8$ samples and loses $1.9$ at $4$ ($12.7$ at $2$,
Appendix~\ref{app:isomargin}); rotation with unmasked replay (dashed) trails it by $0.1$ to
$0.3$, within seed noise, and fails severely in half the seeds at batch $2$; without rotation,
unmasked and isolated replay and the offline night all degrade. \textbf{Right
(timing):} accuracy vs.\ number of waking replay batches. Trickles and clocked bursts fail
alike; pressure-triggered bursts (star) keep most of the accuracy at a third of the events;
surprise triggering fires no events in this regime.}
\label{fig:batch}
\end{figure}

\subsection{Behaviour during the stream}
\label{sec:online}
Final accuracy says what a learner ends up knowing, not what it was doing while it learnt,
and the invariance of Section~\ref{sec:isolation} is a statement about the latter. We
therefore ran the cells of Table~\ref{tab:ablation} again with three measurements taken
along the way (Table~\ref{tab:online}, Figure~\ref{fig:stability}). \emph{Online} accuracy
is scored on each incoming batch before the learner trains on it, which is what an agent
acting on the stream would have got right. \emph{Anytime} and \emph{worst} accuracy come
from evaluating every $50$ waking batches on the classes seen so far, with the suppression
mask off as at the end of training. The \emph{stability gap} is the drop from the accuracy
on the already-learnt tasks just before a task switch to its lowest point within the next
thousand waking batches, averaged over the four switches \citep{delange2023}.
\emph{Churn} is the fraction of the current waking batch whose prediction the replay update
changes, the quantity the Remark in Section~\ref{sec:isolation} reports for the default
system and which we now have for every variant.

\begin{table}[t]
\centering\small
\caption{\textbf{Measured while the stream runs} (split-MNIST, $K{=}1000$, held-out split;
six seeds for the first two rows, three for the rest). Churn is undefined for schedules
that do not replay during the stream.}
\label{tab:online}
\begin{tabular}{@{}lccccc@{}}
\toprule
 & Online & Anytime & Worst & Stability gap & Churn \\
\midrule
isolated replay $+$ rotation & $97.7\pm0.1$ & $94.9\pm0.3$ & $66.2\pm4.4$ & $\mathbf{1.5\pm0.3}$ & $\mathbf{0.33\pm0.04}$ \\
rotation, unmasked replay & $97.8\pm0.1$ & $95.4\pm0.1$ & $69.7\pm2.9$ & $2.3\pm1.1$ & $0.62\pm0.06$ \\
isolated replay, no rotation & $97.6\pm0.1$ & $94.1\pm0.3$ & $63.9\pm1.9$ & $16.9\pm1.8$ & $0.30\pm0.03$ \\
unmasked, no rotation & $95.9\pm1.3$ & $92.9\pm1.4$ & $32.9\pm26.0$ & $15.9\pm7.3$ & $4.39\pm2.15$ \\
offline rehearsal & $87.6\pm0.2$ & $83.3\pm0.7$ & $33.1\pm4.9$ & $38.4\pm0.2$ & --- \\
no replay & $91.3\pm0.7$ & $43.0\pm0.1$ & $18.7\pm0.2$ & $51.8\pm0.0$ & --- \\
BP $+$ DER++, same schedule & $97.0\pm0.1$ & $\mathbf{96.6\pm0.1}$ & $\mathbf{72.7\pm0.1}$ & $2.0\pm0.1$ & --- \\
BP $+$ ER, same schedule & $97.5\pm0.1$ & $96.0\pm0.2$ & $72.6\pm0.8$ & $3.3\pm0.6$ & --- \\
\bottomrule
\end{tabular}
\end{table}

\begin{figure}[t]
\centering
\includegraphics[width=\linewidth]{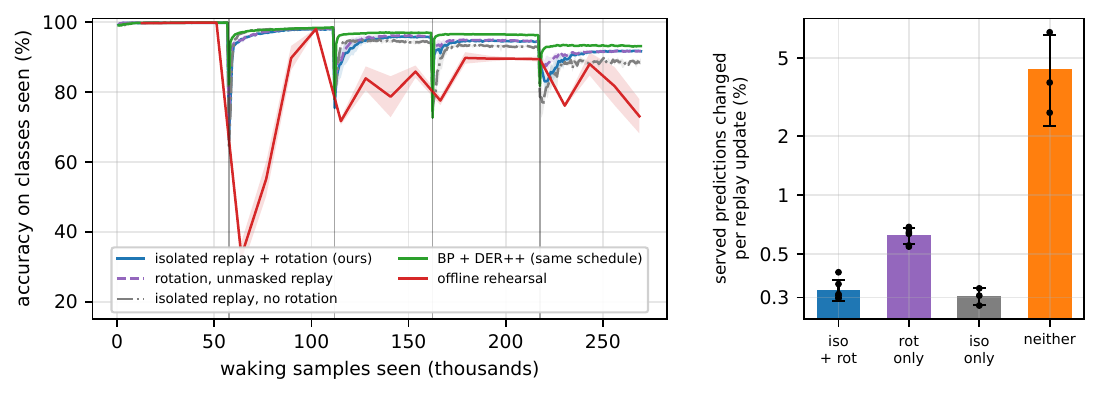}
\caption{\textbf{Left:} accuracy on the classes seen so far, measured every $50$ waking
batches; vertical lines are task switches and bands are one standard deviation over seeds.
Offline rehearsal recovers the old tasks only when its night arrives, so for most of a task
the stream is served by a network that has not yet been repaired. \textbf{Right:} the
fraction of the current batch whose prediction the replay update changes, on a logarithmic
axis, for the four cells of the isolation and rotation square; every seed is drawn.}
\label{fig:stability}
\end{figure}

The two components separate here as cleanly as they do on final accuracy, but along
different axes. Isolation is what holds the served prediction still. The replay update
changes the prediction on $0.33\%$ of the current batch with the mask and on $4.39\%$
without it, and the mask halves the rate even on top of rotation ($0.62\%$; seed-paired
$0.29$ $[+0.26,+0.34]$), which is the practical content of
Propositions~\ref{prop:pre}--\ref{prop:post}: the cells that share the guarantee share the
churn rate, and the cells that do not are one to two orders of magnitude noisier. Rotation
is what keeps the retention curve flat. The stability gap is $1.5$ points with rotation and
about $16$ without it, in both the masked and the unmasked column, and in the figure the
rotating curves recover within a few hundred batches of a task switch while the others sag
for the rest of the task. The offline night is the extreme case of the same effect: it
repairs the old tasks only when it arrives, so it carries a $38$-point gap and serves the
stream at $87.6\%$ against the system's $97.7\%$, which is the cost of an offline phase
stated in the currency an agent cares about.

The honest other half is that the two cells without the mask are better on the retention
axes themselves. Rotation with unmasked replay keeps $0.5$ points more anytime accuracy and
$3.5$ more in the worst case than the isolated system, and DER++ under the same schedule
keeps $1.7$ and $6.5$ more, while serving the stream slightly less accurately ($97.0$
against $97.7$). Writing replay everywhere repairs the old tasks faster, and declining to
do so is what the guarantee costs. Which trade an agent should take depends on whether the
predictions it is making in the meantime are acted on.

\subsection{The direction of protection matters, and the least optimiser state suffices}
\label{sec:direction}
Three controls sharpen the construction. First, \emph{mirror isolation} implements the
classical protected direction with our own machinery: the \emph{waking} update is barred
from synapses whose endpoints both served the last replay batch (protecting the past), while
replay itself is applied unmasked. With everything else identical, mirror isolation reaches
only $79.4\pm3.5\%$, below even unmasked interleaved ER ($87.5\pm0.8$ under Adam) and far
below our direction ($91.6\pm0.3$). Constraining wake learning impairs the new task while
the unmasked replay keeps perturbing the live computation, a double loss. The reversed mask
thus costs twelve points ($20$, with high seed variance, $70.3\pm10.1$, at the $5\%$
fraction in development). It is one implementation of the protect-the-past direction, not a
test of the methods in that literature.
Second, replacing masked Adam with heavy-ball SGD (a single per-synapse velocity and no
second-moment state) inside the same isolation gives $\mathbf{91.6\pm0.3\%}$ sequential and
$94.1\pm0.2\%$ static (six seeds), a tie with masked Adam ($91.5\pm0.4/94.1\pm0.1$; unmasked
SGD control $85.4\pm3.5$). Carrying a single velocity and nothing else is simpler at equal
accuracy and closer to biology; carrying \emph{nothing} is not. With momentum zero the
isolated replay micro-batches lose the velocity's noise averaging, activity collapses to
chance at every learning rate unless the replay step is re-scaled by hand, and the best
pure-SGD setting then trails by $3.2$ sequential and $4.7$ static points
($88.4\pm0.5/89.4\pm0.3$). On the held-out split, confined and leaking moments tie
(Section~\ref{sec:ablation}), so confining the update over time costs nothing.
Third, \emph{soft} rotation (halving, rather than silencing, recent winners) loses on both
axes ($89.9\pm0.7$ sequential, $84.6\pm0.8$ static): all-or-none OFF periods beat turned-down
gain, echoing the biological finding that tonic firing-rate reduction does not discharge
local sleep pressure while ON/OFF alternation does \citep{driessen2026}.
Table~\ref{tab:ablation} in Section~\ref{sec:ablation} collects these controls together with
the remaining component knockouts. A two-sided learning-rate sweep (Appendix~\ref{app:eta})
shows that the optimiser comparison is a plateau rather than a knife-edge, with masked Adam
at its own best rate.

\subsection{When replay must be rare: homeostatic bursts, not surprise}
Figure~\ref{fig:batch} (right): at $12.5\%$ of the replay events, a clocked burst of $16$
batches every $128$ reaches $80.5\pm6.1$, no better than an even trickle at the same budget
($78.5\pm2.6$). Unit-level pressure triggering keeps $88.6\pm1.6$ at a third of the events.
Surprise triggering fails: at the $10\%$ fraction the converged error never crosses the
surprise threshold, so \emph{no} replay event fires ($18.8$, the forgetting floor), and at
the $5\%$ fraction, where it fired occasionally, it still failed at every volume. Homeostatic
triggering therefore outperforms the tested surprise trigger during wake, the opposite of the
night, whose best trigger in our sweeps is novelty-timed.

\subsection{The i.i.d.\ price of rotation}
Rotation's i.i.d.\ price depends on the activity fraction. At the $5\%$ fraction it cost
$1.5$ points on i.i.d.\ MNIST ($92.6$ vs.\ $94.1$ rotation-free; $1.3$--$3.0$ across the
development-phase gate grids) and, in development, \emph{reversed} into a $+2.5$ to
$+6.5$-point static gain on split CIFAR-10, where forced rotation acts as a regulariser in
the underfit regime. At the $10\%$ fraction the MNIST price is $1.3$ points ($94.1$ vs.\
$95.4\pm0.2$ for the same mask without rotation; the unmasked rotation-free control
reaches $93.8\pm5.0$, with one of six seeds collapsed) and the CIFAR static effect is
inconclusive under the available seeds ($28.0\pm1.2$ vs.\ $30.0\pm1.9$). Under masked Adam,
a relative-novelty
gate with the unmastered clause (Eqs.~\eqref{eq:nov}--\eqref{eq:prog}) removed the $5\%$
price with one setting across all four regimes; the full gate study, including why a single
absolute threshold cannot serve across these regimes and why unit-level pressure gating only
trades the axes, is in Appendix~\ref{app:gates}. Under heavy-ball SGD the gate had to be
committed to bouts (Section~\ref{sec:ablation}). At the $10\%$ fraction both remedies become
largely unnecessary, because the activity fraction itself (Section~\ref{sec:dial}) shrinks
the price at its source.

\subsection{Transfer to CIFAR-10}
On split CIFAR-10 (grayscale, same protocol; six seeds) the ordering among the
local-learner variants and BP+ER transfers and widens: isolated replay with rotation
$28.4\pm0.8$ $>$ offline night $25.1\pm2.7$ $>$ BP+ER $24.6\pm0.6$ $>$ unmasked ER
$14.6\pm4.6$ $\approx$ no buffer $16.3$ (the same on the second split:
$29.7>26.5>25.8>15.5$), despite task-active code overlap of $0.9$--$0.98$. The two
cross-entropy references under backprop lead here: ER-ACE $31.6\pm1.0$ and DER++
$30.7\pm0.8$ (seed-paired margins over the system $+3.2$ $[+2.1,+4.4]$ and $+2.2$
$[+1.1,+3.2]$; $31.7$ and $31.0$ on the second split), while A-GEM stays at the no-replay
level ($16.5\pm0.2$). In a single pass the system reaches $26.0\pm1.9$ against $29.5\pm0.8$
for ER-ACE, $25.8\pm0.7$ for BP+ER, $24.4\pm1.2$ for DER++ and $21.6\pm1.0$ for the night:
ER-ACE keeps its lead, DER++ falls behind. On the V1-like
feature front-end the local ordering persists (refractory $42.3\pm0.6$ $>$ offline night
$34.0\pm0.8$ $\gg$ unmasked $20.5\pm1.0$), but BP+ER reaches $42.8\pm0.7$ at its best
learning rate, the reverse of the raw-pixel ordering. At the $5\%$ fraction the gap was
$4.7$ points (local learner at $36.8$, development phase); tuning both sides and raising the
activity fraction reduce it to $0.5$ ($42.3\pm0.6$ local vs.\ $42.8\pm0.7$ BP+ER). Most of
the gap was therefore the local learner's optimiser and its starved activity fraction, not
its learning rule. On raw pixels the local learner leads BP+ER and the offline night
outright and trails only the cross-entropy references, ER-ACE and DER++, by two to three
points.

The residual gap is consistent with the \emph{substrate}, not the learning rule.
Transplanting the local network's ingredients into the BP net one at a time (each cell at its
better of two learning rates): its width \emph{helps} BP+ER ($43.7\pm1.0$ for a dense
$512$--$256$ net); its $30\%$ distance-dependent wiring, drawn by the same generator on the
same sheets, is neutral ($43.5\pm0.2$); and its \kwta{} costs $3$ points ($40.4\pm1.5$ at
$10\%$; $39.5\pm0.6$ at $5\%$), leaving backprop-with-replay \emph{below} the local learner
on the identical substrate ($42.3\pm0.6$ vs.\ $40.4\pm1.5$ at $10\%$; three seeds each).
What separates the two systems on feature inputs is dense activation, not backpropagation's
credit assignment: on the substrate used here, the decomposition leaves no gap attributable
to the learning rule. The advantage over \emph{other local schedules} (nights, unmasked ER)
holds across all three input regimes.

\subsection{The mechanism on a backprop learner}
\label{sec:bpk}
Nothing in the construction requires the local rule: Propositions~\ref{prop:pre}--\ref{prop:post}
concern the forward computation, not the way the update was computed. Table~\ref{tab:bpk}
therefore runs the schedule of Section~\ref{sec:isolation} unchanged on a backprop MLP of the
same width with \kwta{} hidden layers at the same $10\%$ fraction: waking batches of $16$
with an unmasked Adam step, awake sets read from the waking forward pass, one $16$-sample
replay micro-batch per waking batch with the synapse mask of \eqref{eq:mask} applied to the
weight change and to Adam's moments (the readout unmasked, the input treated as always
awake), and the refractory rotation as a suppression mask on the next forward pass. The
learning rate is selected on the development split per variant (Section~\ref{sec:cost}).

\begin{table}[t]
\centering
\caption{\textbf{Backprop learners under our schedule} ($512$--$256$, Adam with weight decay
$10^{-3}$, waking batches of $16$ with one $16$-sample replay micro-batch each, $K{=}1000$;
held-out split, six seeds; the learning rate of every row re-selected on the development
split). Upper block: \kwta{} hidden layers at the same $10\%$ fraction as the local learner.
Middle block: the published baselines of Table~\ref{tab:ablation}, unchanged except for the
schedule. Lower block: the local system, for reference.}
\label{tab:bpk}
\small
\begin{tabular}{@{}lcccc@{}}
\toprule
Backprop, $512$--$256$ & lr & Sequential & i.i.d.\ (static) & Seq., single pass \\
\midrule
$+$ \kwta{}, interleaved replay (unmasked) & $3{\cdot}10^{-5}$ & $91.7\pm0.4$ & $97.2\pm0.1$ & $89.3\pm0.6$ \\
$+$ \kwta{}, isolated replay & $10^{-4}$ & $84.7\pm2.6$ & $97.6\pm0.1$ & $81.9\pm2.1$ \\
$+$ \kwta{}, interleaved replay $+$ rotation & $10^{-4}$ & $92.2\pm0.3$ & $96.7\pm0.3$ & $92.7\pm0.3$ \\
$+$ \kwta{}, isolated replay $+$ rotation & $10^{-4}$ & $92.1\pm0.3$ & $96.6\pm0.1$ & $92.8\pm0.2$ \\
\midrule
ER & $3{\cdot}10^{-5}$ & $90.8\pm0.6$ & --- & $87.8\pm0.6$ \\
ER-ACE & $3{\cdot}10^{-4}$ & $89.7\pm0.4$ & --- & $90.5\pm0.6$ \\
DER++ & $3{\cdot}10^{-5}$ & $\mathbf{93.1\pm0.6}$ & --- & $90.2\pm0.4$ \\
\midrule
local learner, full system (Table~\ref{tab:ablation}) & & $91.6\pm0.3$ & $94.1\pm0.2$ & $\mathbf{91.8\pm0.3}$ \\
\bottomrule
\end{tabular}
\end{table}

Rotation transfers. It lifts interleaved replay from $91.7\pm0.4$ to $92.2\pm0.3$
(seed-paired $+0.5$ $[+0.2,+0.8]$) and from $89.3$ to $92.7$ in a single pass ($+3.4$
$[+2.8,+4.0]$), to the level of the local learner ($+0.6$ $[+0.3,+1.0]$ paired, same
schedule), at a static price of $0.5$--$0.9$
points, smaller than the local learner's $1.3$. Isolation on top of rotation again costs
nothing ($-0.1$ $[-0.3,+0.1]$;
single pass $+0.1$) and brings the guarantee. Isolation alone, however, hurts the backprop
learner ($84.7\pm2.6$, $-7.0$ $[-8.9,-5.2]$ against interleaved replay) where it helped the
local one ($+2.6$, Table~\ref{tab:ablation}), although the natural silence is comparable
($0.63$ and $0.81$ of the units asleep per layer against $0.7$--$0.8$ on the local
learner): the channel that natural silence leaves open is too narrow for this learner, and
on backprop the guarantee is affordable only through the rotation. The backprop learner also
prefers a far smaller step for the unrotated variant ($3{\cdot}10^{-5}$) than for the
rotating ones ($10^{-4}$), which is the optimiser-side signature of the same effect: without
rotation, interleaved replay survives only by moving slowly.

The lower block separates the schedule from the sparsity, by re-tuning and re-running the
published baselines of Table~\ref{tab:ablation} under the micro-batch schedule and nothing
else. Moving to it is worth $2.0$ points to ER ($88.8\to90.8$) and $1.6$ to DER++
($91.5\to93.1$) and costs ER-ACE $0.5$ ($90.2\to89.7$), so the schedule accounts for part of
the margin that Table~\ref{tab:ablation} reports and \kwta{} for the rest of ER's distance
to $91.7$. Measured within the schedule, the local system keeps a lead of $0.9$
$[+0.2,+1.6]$ over ER and $1.9$ $[+1.4,+2.4]$ over ER-ACE and gives up $1.5$
$[+1.0,+2.0]$ to DER++, which at $93.1\pm0.6$ is the strongest cell we measured on
split-MNIST. The single-pass column turns that around: there the local system leads DER++ by
$1.6$ $[+1.2,+2.0]$, ER-ACE by $1.3$ $[+0.9,+2.0]$ and ER by $4.0$ $[+3.5,+4.4]$, and DER++
loses $2.9$ points relative to its own five-epoch result while the local system loses none.
Seeing the data once is the regime the construction was built for, and it is the regime in
which it wins; seeing it five times is the regime in which distilling stored logits pays
more than protecting the live computation.

\subsection{Energy}
Converted to spiking inference ($T_s{=}8$; firing thresholds calibrated on training images,
evaluated on the held-out tenth), the default system (always-on rotation, heavy-ball SGD,
$10\%$ fraction; rate accuracy $94.2\%$ on the measured seed) runs at $92.7\%$ and
$117$\,nJ/sample vs.\ $2461$\,nJ for dense rate inference ($21\times$) and $526$\,nJ for
event-driven rate inference ($4.5\times$), improving with the time budget ($94.0\%$ at
$T_s{=}16$, $236$\,nJ; $94.0\%$ at $T_s{=}32$). The rotation-free control converts to
$94.2\%$ at $T_s{=}8$ and dips slightly to $93.9\%$ at $T_s{=}32$; the interaction between
rotation and spike calibration is therefore small under heavy-ball SGD on the held-out
split.

\subsection{Negative results}
The following alternatives were tested under matched budgets and rejected: a frozen
dentate-gyrus-style sparse expansion front-end (input decorrelation does not propagate
through learned \kwta{} layers: task overlap at the input drops $0.79\to0.18$, yet every
downstream metric worsens by $3$--$11$ points); generative (REM) replay as a substitute for
episodic samples; margin-based reverse learning; surprise-gated memory \emph{writing};
multiplicative burst coding; absolute-threshold rotation gates; pruning anchored to sparse
replay bursts (Appendix~\ref{app:downsel}); soft (gain-reduction) rotation and
mirror-direction isolation (Section~\ref{sec:direction}); an OFF-side gate refractory, the
inactive half of the flip-flop in this system; strongly coupled synaptic anchors; and
utility-exempt rotation as a route past the bout gate (all Section~\ref{sec:ablation}).

\section{Ablation Study}
\label{sec:ablation}

The system under test combines five ingredients: a biologically constrained substrate, the
isolation mask \eqref{eq:mask}, the refractory rotation, a heavy-ball optimiser, and the
gate \eqref{eq:prog}. This section removes or replaces one ingredient of the consolidation
mechanism at a time (Table~\ref{tab:ablation}). The substrate beneath it was fixed by earlier
ablation ladders; its cost accounting ($1.2$ points below a dense backpropagation MLP at
$4\times$ fewer synaptic operations per pass) is reproduced in Appendix~\ref{app:substrate}.
Every mechanism cell shares the default substrate ($512$--$256$ hidden, $10\%$ \kwta{},
$30\%$ distance-dependent connectivity), the same buffer ($K{=}1000$, random writing) and the
same replay budget (waking batch $16$, replay batch $16$ after every waking batch). The
appendices report the corresponding grids at the $5\%$ fraction, where each ordering below
was first established.

\begin{table}[t]
\centering
\caption{\textbf{Component ablation of the consolidation mechanism} on the default substrate
($512$--$256$/$10\%$; backprop references dense, at the widths selected in
Section~\ref{sec:cost}). Sequential $=$ final accuracy
on split-MNIST; static $=$ i.i.d.\ MNIST (the axis that must not degrade). Dashes:
combination not applicable; blank: not run. Single pass $=$ one epoch per task, otherwise
identical. Held-out split; six seeds, except the gate, anchor, optimiser and unmasked-Adam
rows (three).}
\label{tab:ablation}
\footnotesize\setlength{\tabcolsep}{4pt}
\begin{tabular}{@{}lccc@{}}
\toprule
Variant & Sequential & i.i.d.\ (static) & Seq., single pass \\
\midrule
\textbf{full system}: isolated replay $+$ always-on rotation (heavy-ball SGD) & $\mathbf{91.6\pm0.3}$ & $94.1\pm0.2$ & $\mathbf{91.8\pm0.3}$ \\
\quad rotation gated by relative novelty, committed to bouts & $91.8\pm0.3$ & $94.4\pm0.3$ &  \\
\quad rotation gated per batch & $87.5\pm1.6$ & $94.7\pm0.1$ &  \\
$+$ two-timescale anchor ($\lambda{=}\mu{=}3{\cdot}10^{-4}$), rotation always on & $91.7\pm0.0$ & $93.8\pm0.3$ &  \\
$-$ rotation (same mask, Eq.~\ref{eq:mask}, natural silence only) & $88.0\pm0.7$ & $95.4\pm0.2$ & $86.4\pm1.0$ \\
$-$ isolation, $-$ rotation (unmasked interleaved replay) & $85.4\pm3.5$ & $93.8\pm5.0$ & $83.0\pm14.9$ \\
$-$ isolation, $+$ rotation (rotation alone) & $91.5\pm0.3$ & $93.7\pm0.3$ & $91.5\pm0.8$ \\
readout-only replay (hidden layers take no replay update) & $41.7\pm10.9$ & $81.2\pm2.5$ & $25.1\pm3.0$ \\
random rotation (matched suppression count) & $91.1\pm0.7$ & $92.3\pm0.3$ & $90.5\pm0.9$ \\
$-$ replay (no buffer) & $19.4\pm0.0$ & $95.1\pm0.1$ & $18.1\pm0.5$ \\
\midrule
masked Adam, moments confined to the mask & $91.5\pm0.4$ & $94.1\pm0.1$ &  \\
masked Adam, moments leaking outside the mask & $91.6\pm0.6$ & $94.0\pm0.3$ &  \\
momentum $0$ (pure SGD; $\eta$ and replay gain re-tuned, else chance) & $88.4\pm0.5$ & $89.4\pm0.3$ &  \\
\midrule
offline night instead of concurrent replay, same substrate & $89.0\pm3.5$ & --- & $76.9\pm3.4$ \\
\quad the same night on its preferred narrow substrate & $88.5\pm2.8$ & --- & $71.5\pm1.9$ \\
mirror direction: protect the past & $79.4\pm3.5$ & $90.2\pm3.0$ & $80.9\pm4.1$ \\
soft rotation: gain $0.5$ instead of $0$ & $89.9\pm0.7$ & $84.6\pm0.8$ & $87.8\pm1.4$ \\
\midrule
unmasked ER at the same budget (Adam) & $87.5\pm0.8$ & $95.1\pm0.2$ &  \\
backprop $+$ ER (plain backprop on the static axis), dense & $88.8\pm0.3$ & $97.4\pm0.1$ & $88.6\pm0.2$ \\
backprop $+$ DER++ & $91.5\pm0.5$ & --- & $90.1\pm0.7$ \\
backprop $+$ ER-ACE & $90.2\pm0.4$ & --- & $88.5\pm0.3$ \\
backprop $+$ A-GEM & $68.5\pm7.0$ & --- & $46.8\pm5.0$ \\
\bottomrule
\end{tabular}
\end{table}

\subsection{Reading the component table}
The first block prices each ingredient by what its removal costs the sequential axis under
heavy-ball SGD. Rotation is worth $+3.6$ points ($91.6$ vs.\ $88.0$ for the same mask on
natural silence alone), rotation and isolation together $+6.2$ with a twelve-fold smaller
seed deviation ($91.6\pm0.3$ vs.\ $85.4\pm3.5$), isolation alone $+2.6$ ($88.0$ vs.\
$85.4$), isolation on top of rotation $+0.1$ ($91.5\pm0.3$ for rotation with unmasked
replay; $[-0.3,+0.5]$ paired), and replay is worth everything ($19.4$ is the forgetting
floor). Replay confined to the readout row, with the hidden layers frozen for replay, falls
to $41.7\pm10.9$ (ten classes) and costs $13$ static points: readout-only replay cannot
account for the full system's accuracy, so hidden-layer replay updates are essential to it.
A random suppression of matched size in place of the refractory rule recovers most of
rotation's gain ($91.1\pm0.7$) but trails it by $0.5$ sequential and $1.8$ static points
with $1.3\times$ the forgetting ($F{=}8.3$ vs.\ $6.4$): alternating the coalition carries
most of the effect, and resting the recent winners the rest. The static column shows which
component pays: rotation is the only statically costly one ($94.1$ always-on against
$95.4$ for the same mask without it, and $93.7$ with rotation and no isolation), a price of
$1.3$ points at the $10\%$ fraction; the unmasked rotation-free control's static cell
($93.8\pm5.0$) hides one collapsed seed among five at $95.4$--$96.0$. The single-pass column
repeats the sequential ordering with the margins widened: every variant without rotation
loses more in one pass than in five, and the offline night most of all. The unmasked
control is not handicapped by the
replay step ($3\eta$ against a waking step of $\eta/16$). Sweeping the replay gain over
$\{1/16,1/4,1,3,10\}$ moves masked replay $67.2\to86.8\to90.8\to91.6\to74.2$ and unmasked
replay $61.6\to81.4\to85.7\to85.4\to47.3$ (sequential; three seeds, six at gain $10$, where
both destabilise): masked replay peaks at the default, unmasked replay plateaus from gain $1$
to $3$, and the rotation-free unmasked control never closes the gap. The sweep bounds that
control; isolation's own contribution is estimated by the square above. The gates that were
built to refund a larger price at the $5\%$ fraction are neutral here: the bout-committed
gate trades $+0.2$ sequential for $+0.3$ static ($91.8\pm0.3/94.4\pm0.3$), the per-batch gate
is erratic ($87.5\pm1.6/94.7\pm0.1$), and the two-timescale anchor no longer adds anything
($91.7\pm0.0/93.8\pm0.3$). Their mechanisms remain instructive. At $5\%$ the per-batch gate's
flicker destroyed heavy-ball SGD's gain; committing the gate to bouts (loosely modelled on
the consolidated bouts of the sleep--wake switch \citep{saper2005}, which motivates commitment
but not the specific minimum duration) repaired it; and bout-length and OFF-refractory
controls located the residual price in the rotation of settled coalitions itself
(Appendix~\ref{app:probes}). At the $10\%$ fraction, however, the larger activity fraction
has absorbed most of what the gates bought, and the default system is simply always-on
rotation.

The middle block tests the confinement requirement itself.
Propositions~\ref{prop:pre}--\ref{prop:post} require Adam's moments to advance only inside
the mask. State advancing outside re-injects the replay gradient into awake synapses through
the next unmasked waking step: the present batch's invariance still holds, since the weight
change itself stays masked, but the confinement of the update over time is lost.
Empirically, on the held-out split, confined and leaking moments tie ($91.5\pm0.4$ vs.\
$91.6\pm0.6$), as does the single velocity ($91.6\pm0.3$). The velocity is kept as the least
state that stays confined to the mask at equal accuracy; removing it as well (momentum zero,
Table~\ref{tab:ablation}) costs $3.2$/$4.7$ points once the replay step has been re-scaled
to survive at all.

One budget note completes the block: half the cadence with $16$-sample batches costs $0.4$
($91.2\pm0.2$), so the system prefers consolidation to run continuously.

\subsection{Two probes of the residual static price}
At the $5\%$ fraction, the OFF-refractory control located the bout gate's residual static
price in the rotation of settled coalitions itself. Two mechanisms addressed that diagnosis
directly (details in Appendix~\ref{app:probes}; $5\%$ numbers). \emph{Utility-exempt
rotation}, which spares the top-$q$ long-use units from ever being benched, only traces a
trade-off that the bout gate dominates: temporal commitment beats structural exemption.
\emph{Two-timescale synapses} in the spirit of synaptic consolidation cascades
\citep{benna2016} do not dissolve the price, but weak symmetric coupling acts as a
\emph{sequential stabiliser} at $5\%$ ($92.4\pm0.2$, variance halved, static parity); at the
$10\%$ fraction the anchor is neutral ($91.7\pm0.0/93.8\pm0.3$). Anchor and bout gate compose
sub-additively. Figure~\ref{fig:frontier} draws the operating frontier.

\begin{figure}[t]
\centering
\includegraphics[width=.6\linewidth]{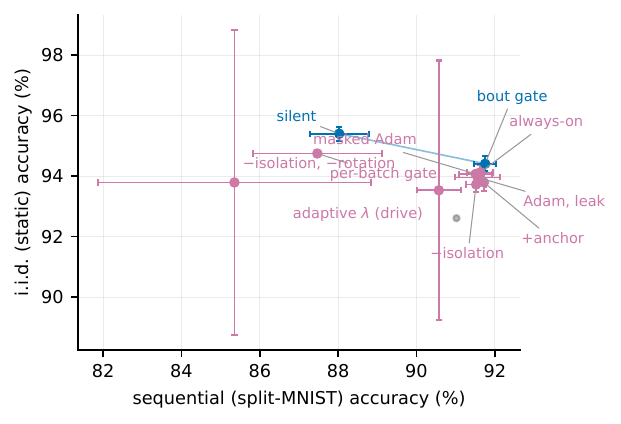}
\caption{\textbf{The operating frontier} at the $10\%$ fraction under heavy-ball
SGD: split-MNIST sequential accuracy against i.i.d.\ static accuracy (headline points
six seeds). Blue: non-dominated by seed means (the bout-committed gate and the silent mask);
purple: dominated; held-out split. The top cluster (always-on rotation, bout gate, anchor,
masked Adam) lies within seed noise of one another; the frontier trades sequential for
static accuracy through the silent mask, the unmasked rotation-free control is dominated by
it once its collapsed seed is counted, and the adaptive-$\lambda$ point
(Section~\ref{sec:controller}) is dominated on this split.}
\label{fig:frontier}
\end{figure}

\begin{figure}[h]
\centering
\includegraphics[width=.48\linewidth]{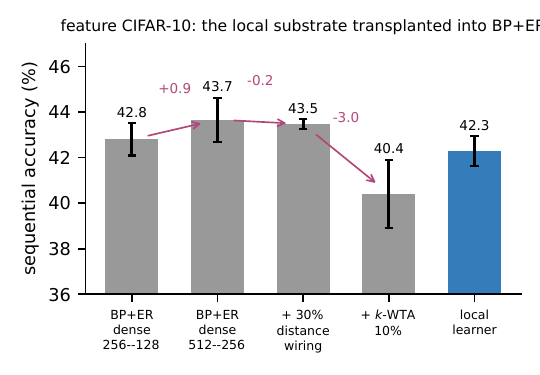}\hfill
\includegraphics[width=.48\linewidth]{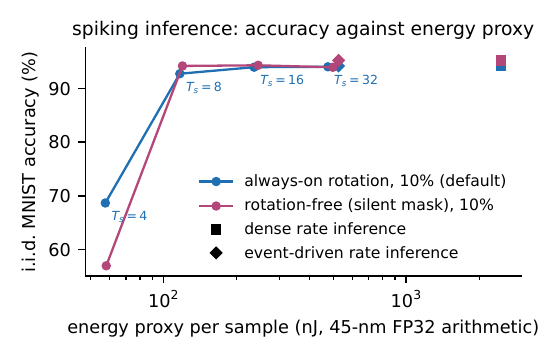}
\caption{\textbf{Left:} the substrate decomposition on feature CIFAR-10, transplanting the
local network's ingredients into BP+ER one at a time (each cell at its better of two learning
rates; three seeds). \textbf{Right:} spiking inference,
i.i.d.\ MNIST accuracy against the idealised 45-nm FP32 arithmetic-energy proxy for
$T_s\in\{4,8,16,32\}$, with the dense (square) and event-driven (diamond) rate references;
held-out split; the default is monotone in $T_s$, the rotation-free control dips $0.4$ at
$T_s{=}32$.}
\label{fig:appfigs}
\end{figure}

\subsection{The activity fraction}
\label{sec:dial}
The substrate's $5\%$ \kwta{} fraction was fixed early in the project. The decomposition
above suggested that it was too low: if sparse \kwta{} costs BP several points on features,
it presumably costs our own learner too. It does, on both regimes, and the cost is recovered
almost for free; this subsection documents the sweep that led to the $10\%$ fraction used
throughout the main text. On feature CIFAR-10, relaxing the fraction improves \emph{both}
axes monotonically: sequential $40.9\to42.9\pm0.6$ and static $38.3\to46.4\pm1.6$ from $5$
to $25\%$, with forgetting slightly \emph{reduced}. The BP control traces the same curve
($39.5\to40.4\to40.7$ at $5$, $10$, $15\%$), with the local learner $1.4$--$1.9$ points ahead
at every matched level. On MNIST the sequential axis is flat from $10\%$ upward: always-on
rotation reaches $\mathbf{91.6\pm0.3}$ sequential at $94.1\pm0.2$ static (six seeds), $15\%$
gives $91.4\pm0.4/94.4\pm0.2$ and $25\%$ gives $91.5\pm0.4/94.6\pm0.2$, against
$91.0\pm0.5/92.6\pm0.0$ at $5\%$.

Three controls sharpen the reading. First, the no-buffer ceiling itself rises with the
fraction ($94.1$ at $5\%$, $95.1$ at $10\%$), so rotation's static price \emph{at matched
sparsity} shrinks from $1.5$ to $1.0$ points, and the need for a gate shrinks with it: at
$10\%$ always-on rotation no longer trails the bout-gated system. Second, in development,
rotation-free silent isolation gained nothing from the larger fraction
($89.0\to89.3\to88.7$, official test set): the change helps specifically the
\emph{rotating} family, which at $5\%$ was starved of representational room. Third,
asymmetric schedules (a mild input-facing layer above a sparse deep layer) failed in
development, with sequential parity to uniform relaxation but a $3$--$4$-point static
deficit: the cost is levied at every layer. The larger fraction is free elsewhere, since
synaptic operations rise only $3$--$5\%$. The default substrate therefore uses $10\%$;
Table~\ref{tab:ablation} and every main-text number use it, with the complete mechanism
grid re-established there and all orderings preserved, while the appendices report the
$5\%$ grids on which each mechanism was first isolated. The activity fraction, not the
gate, is the cheapest remedy for rotation's static price (Figure~\ref{fig:frontier}).

\section{An adaptive decay controller}
\label{sec:controller}

\paragraph{The last tuned constant.}
Every result above runs the rule of \eqref{eq:kp} with a decay $\lambda$ tuned once per
dataset ($10^{-3}$ on the image benchmarks). Split CIFAR-100, ten tasks of ten classes,
exposes what that constant hides. With one hundred classes, each class receives a tenth of
the per-epoch supervision events, while the decay acts on every synapse at every step,
undiminished. At the MNIST-tuned value the hidden-to-hidden weights follow the pure-decay
trajectory $(1-\lambda)^t$ \emph{exactly}, because the thinned error signal loses the race
outright; deep activity collapses by $5\times$, the readout gradient goes to zero, and the
network lands at chance, $2.4\%$ i.i.d.\ against $20$--$23\%$ for a ten-class subset of the
same images on the same substrate (development-phase diagnosis). The class count of the
supervision signal is a hidden variable that every fixed $\lambda$ is silently tuned against.
A manual rescue ($\lambda{=}10^{-4}$, $3\times$ targets) revives activity but survives the full
protocol only at $1.0$--$1.2\%$.

\paragraph{The controller.}
Our design principle that control signals must be self-referenced rather than absolute
(Section~\ref{sec:rotation}) applies at the synapse as well. We replace the constant by a
per-layer ratio controller,
\begin{equation}
\lambda_\ell \;=\; \min\!\Big(\lambda_{\max},\; \rho\,
\frac{\mathrm{EMA}\big[\,u^\ell\,\big]}{\overline{|W^\ell|}+\epsilon}\Big),
\qquad \rho=0.25,\;\; \lambda_{\max}=0.02,
\label{eq:controller}
\end{equation}
where $u^\ell$ is the mean magnitude of the data-driven, post-optimiser, \emph{pre-decay}
waking increment applied to layer $\ell$ (so the decay never feeds its own drive), the EMA
(coefficient $0.02$) runs over waking, unmasked steps only, $\epsilon=10^{-12}$ floors the
denominator, and $W^\ell$ and $B^{\ell-1}$ share $\lambda_\ell$, so Kolen--Pollack alignment
is untouched. The decay's mean magnitude per step, $\lambda_\ell\overline{|W^\ell|}$, is held at
the fraction $\rho$ of the mean data-driven update magnitude: layers and regimes with rich
supervision keep a strong regulariser, starved ones are automatically spared. $\rho$ is the
only control target; $\lambda_{\max}$, the EMA coefficient and $\epsilon$ are global
safeguards never tuned per dataset. Isolation and rotation decide \emph{where} and
\emph{when} to consolidate; the controller decides \emph{how strongly} to update as the
scale changes. The same principle underlies every control signal in this paper: novelty is a
fast-over-slow error ratio, homeostasis a unit's own use, decay a drive-over-weight ratio,
all referenced to the learner's own current state and none absolute. A gradient-referenced
variant ($\rho\,\eta\,\mathrm{EMA}[\,\overline{|g^\ell|}\,]/\overline{|W^\ell|}$) behaved
equivalently in development (official test set, not re-run under the held-out protocol); we
use the drive-referenced form because it is optimiser-agnostic.

\begin{table}[t]
\centering\small
\caption{\textbf{One dataset-independent ratio target replaces per-dataset decay tuning.} Sequential /
static accuracy (\%) under the tuned fixed decay ($10^{-3}$ on the image benchmarks;
$10^{-4}$ with $3\times$ targets as the CIFAR-100 rescue) against the drive-referenced
controller of \eqref{eq:controller} at the same $\rho=0.25$ everywhere; held-out split, three
seeds (six on split-MNIST and in the fixed-decay split CIFAR-10 cell); the full system
(refractory rotation, isolated replay, $K{=}1000$) in every cell.}
\label{tab:controller}
\begin{tabular}{lcccc}
\toprule
& \multicolumn{2}{c}{fixed $\lambda$ (tuned)} & \multicolumn{2}{c}{controller (drive)} \\
& seq. & static & seq. & static \\
\midrule
split-MNIST & $91.6\pm0.3$ & $94.1\pm0.2$ & $90.6\pm0.6$ & $93.5\pm4.3$ \\
split CIFAR-10 & $28.4\pm0.8$ & $28.0\pm1.2$ & $26.2\pm1.1$ & $30.1\pm0.2$ \\
feature CIFAR-10 & $42.3\pm0.6$ & $42.4\pm1.1$ & $39.1\pm0.4$ & $43.6\pm0.6$ \\
split CIFAR-100 & $1.2\pm0.1$ & $1.0\pm0.1$ & $3.2\pm0.8$ & $4.1\pm0.6$ \\
\quad + features & $1.1\pm0.3$ & $1.1\pm0.2$ & $5.2\pm1.6$ & $6.4\pm0.6$ \\
\bottomrule
\end{tabular}
\end{table}

\paragraph{One ratio target, every dataset.}
Table~\ref{tab:controller} gives the cross-dataset comparison. On the sequential axis the
controller concedes $1$--$3$ points to the per-dataset tuned value. On the static axis it
gains $1.2$--$5.3$ points on CIFAR-10, feature CIFAR-10 and both CIFAR-100 variants; on
split-MNIST it gains $1.1$ on the second held-out split ($95.2\pm0.4$ vs.\ $94.1\pm0.6$) but
one of six seeds collapsed on the first ($93.5\pm4.3$ vs.\ $94.1\pm0.2$), so we claim no
static gain there. On CIFAR-100 the controller lifts the system $2.7$--$4.7\times$ over the
manual rescue with no dataset-specific decay coefficient (floor-regime numbers, reported for
ordering only; the BP+ER reference reaches $6.8$ raw and $10.1$ with features). The realised
$\lambda_\ell$, logged throughout, differentiates exactly as the tuning history would
predict: $\sim\!10^{-5}$ on MNIST sequential streams, $\sim\!10^{-4}$ on static ones, up to
$3\times10^{-3}$ on raw CIFAR, and layer-graded within each net. The per-dataset $\lambda$
ladder was therefore a signal-to-decay ratio in disguise. On the operating frontier
(Figure~\ref{fig:frontier}) the controller point is dominated by the default on the first split
($90.6/93.5$ against $91.6/94.1$): on MNIST the tuned fixed decay remains the better setting on
both axes, and the controller's value lies on harder inputs and at depth.

\paragraph{Depth is the same boundary, and two mechanisms cross it.}
The class count thins the error signal per class; depth attenuates it per layer. This is the
same starvation on a different axis. A depth ladder (three, four and five hidden layers,
$512$--$256$--$128\ldots$, everything else the default configuration) makes the parallel
exact. Under the tuned fixed decay without skips, the static axis dies at four hidden layers
($10.1\pm0.2\%$, chance) and the sequential axis collapses at five ($26.3\pm20.2$); under the
controller every cell stays alive and degrades gracefully ($90.7\to88.9\to88.1$ sequential,
$95.3\to94.8\to89.6$ static), with the realised $\lambda_\ell$ dropping to $3\times10^{-6}$
in the starved middle layers. The residual toll is then architectural, and a biologically
unexceptional bypass removes it. ResNet-style skip synapses $S^\ell$ carry $a^{\ell-2}$ into
each deep hidden layer; they are learned by the same local product with a Kolen--Pollack
feedback twin and the shared per-layer decay, and replay-masked under the same
pre-or-post-asleep rule (Appendix~\ref{app:alg}). They restore the thick error path, and the
five-hidden-layer system reaches $91.0\pm0.3$ sequential and $95.8\pm0.3$ static (six seeds),
against the two-layer default's $91.6/94.1$ and the two-layer controller's $90.6/93.5$: $0.6$
sequential points below the default and $1.7$ above it on the static axis, with $95.8\pm0.1$
at depth four. The bypass alone also rescues the fixed decay ($90.4/93.5$ at depth five),
which indicates that error-path attenuation, rather than decay mis-tuning as such, is a
principal contributor to the depth collapse; the controller still adds its usual static
margin on top of the bypass. We did not re-tune the fixed decay per depth: the claim is that
the controller needs no re-tuning, not that no fixed value would serve. A BP+ER reference at
the same widths is depth-insensitive ($88.9\to87.9$ sequential) but forgets $2\times$ more
($F\approx13$--$14$ vs.\ $7$--$10$); the local system with skips and the controller beats it
at every depth tested. The two starvations are not interchangeable, however. On split
CIFAR-100 depth is a net loss that skips only half refund (development-phase numbers: four
hidden layers $1.9/1.6\to3.8/3.0$ with skips, against $5.7/6.0$ at two), whereas \emph{width}
at the short chain length lifts the local system to $5.8\pm0.9/7.7\pm1.1$ ($1024$--$512$).
The thin-signal regime wants capacity, not depth, and BP+ER's remaining lead there
($10.1\pm0.4$) comes with twice the forgetting.

\paragraph{What the controller must never see, and what it does not do.}
Composition follows one rule: the drive estimate must be fed only \emph{unamplified waking
updates}. Inflated one-hot targets and the offline night's $3\times$-gain batches both push
the EMA up until $\lambda_\ell$ hits its clip and erases day learning (in development, night
$\times$ controller collapsed to $60.2\pm2.8$ against $90.9\pm3.2$ for the night alone),
whereas mechanisms that respect this rule compose cleanly (anchor and bout gate under the
controller kept their usual signature in development). The controller also deliberately does
not take over $\lambda$'s second role. On small tabular problems the decay serves as a
\emph{capacity} regulariser against overfitting (cf.\ Appendix~\ref{app:substrate}),
information that no drive ratio contains; where generalisation rather than alignment sets
the correct decay, a tuned constant should keep the job.

\section{Discussion}
\label{sec:discussion}

\paragraph{Why micro-batches work.} A replay micro-batch is a high-variance gradient
estimate. Applied to a static coalition it perturbs the very units encoding the wake stream,
and with shared optimiser state the perturbation compounds (the rotation-free unmasked curve
of Figure~\ref{fig:batch}). Two things defuse it. Alternating the coalition, which is what
the rotation does, is where most of the accuracy comes from: with it, even unmasked replay
holds $91.5$ at batch $16$ and $89.3$ at batch $4$. Confining the update to the mask adds
$0.1$--$0.3$ points at those batch sizes, shows none of the severe failures that unmasked
replay shows in half the seeds at batch $2$ (Appendix~\ref{app:isomargin}), and gives the
guarantee that the present input's hidden computation, under its suppression mask, is
untouched: exactly on the proven channels, and for all but $\approx0.3\%$ of waking samples
per step in measurement (Propositions~\ref{prop:pre}--\ref{prop:post} and the Remark). The
offline night has neither protection: with nothing clamping the network, a small noisy batch
at $3\times$ learning rate moves the whole weight state.

\paragraph{Prerequisites beyond this substrate.}
The construction is a recipe with two checkable prerequisites. The representation must be
sparse enough that a batch leaves most hidden units silent, and the replay update must be
maskable synapse by synapse. Nothing in Propositions~\ref{prop:pre}--\ref{prop:post}
depends on how the update is computed; we studied a local learner because its sparse codes
and local rules make both prerequisites hold by construction, and Section~\ref{sec:bpk}
shows the same schedule on a backpropagation learner with \kwta{} layers: rotation
transfers, isolation costs nothing on top of it, and isolation alone does not transfer.
Under these prerequisites, consolidation can run in the silent
degrees of freedom between the
steps of the stream, with no separate offline phase, under a conditional guarantee on the
computation in use, and at a higher total compute ($6\times$ the night's wall-clock,
unoptimised).

Three things decide whether the prerequisites survive on a modern stack, and all three can
be measured before any continual-learning run. The first is the waking batch. The awake set
is a union over its samples, so the channel narrows as the batch grows: at the $10\%$
fraction and independent activity, $16$ samples would leave $18.5\%$ of a layer silent,
whereas the measured fractions run from $0.51$ to $0.83$ across layers and learners, because
units specialise and our batches are drawn from one task at a time. Our waking batch holds $16$ samples; what the guarantee has
to protect in a deployed agent is the handful of observations it is serving, not the batch
that training throughput would prefer, and defining the mask on the served inputs keeps the
union from growing with the training batch. The second is weight sharing. A convolutional
filter is shared across positions and an attention or feed-forward weight across tokens, so
a synapse is invisible only if its unit is silent at every position and every token, which
asks for channel-level or neuron-level sparsity rather than the elementwise kind. The third
is normalisation. A normaliser that pools over the units \kwta{} selects makes one asleep
unit's pre-activation visible to all of them and voids
Proposition~\ref{prop:post}; the feed-forward block of a pre-norm transformer does not
normalise over its hidden units, and per-channel batch normalisation does not pool across
channels, although replay passes would still move its running statistics. Top-$k$ routed
mixture-of-experts layers expose, per token, the experts a token leaves silent, but whether
the routes of a whole batch leave enough experts idle is a measurement, since load balancing
pushes against it.

\paragraph{Correspondence with biology.}
We treat the biology as a motivating analogy, not as a claim of mechanistic equivalence. The
refractory rule mirrors the causal finding of \citet{driessen2026} that ON/OFF
\emph{alternation}, not tonic reduction, discharges local sleep pressure and rescues memory:
our soft-rotation control (a tonic reduction) is reliably worse, and the rotation-free
silent mask, with no imposed off-periods at all, is worse still. The dissociation of
triggers, homeostatic for waking consolidation and novelty-timed for full nights, is a
testable prediction for biology. The regime dependence of the rotation price (a cost on easy
streams, a gain under chronic underfitting) may bear on why local sleep in vivo both impairs
behaviour \citep{vyazovskiy2011} and serves useful functions when appropriately timed.

\paragraph{Limitations.}
Configurations were selected on the official test splits (the development set) and frozen;
every main comparison re-trains them with a held-out tenth of the training data removed from
the stream and evaluates on it (two splits for the headline rows, sharing $9\%$ of their
samples), and exploratory results not re-run under this protocol are labelled
development-phase. Development-set numbers sit $1.1$ points above the held-out ones on
average and rank-correlate with them at Spearman $\rho=0.97$, a robustness check rather
than a proof that selection bias is absent (Appendix~\ref{app:val}). The protocol is five
epochs per task; a single pass keeps the lead (Table~\ref{tab:ablation},
Appendix~\ref{app:probes}). All results are at MLP scale on MNIST/CIFAR variants, with six
seeds on the component table, the buffer-size axis and the CIFAR-10 comparison and three
elsewhere. The backpropagation references are external references rather than a
state-of-the-art claim, and in Table~\ref{tab:ablation} they run under their own batch-$256$
protocol. The schedule is worth up to two points on its own
(Section~\ref{sec:bpk}), so that table compares across schedules; the within-schedule
reading is Table~\ref{tab:bpk}, where the system keeps its lead over ER and ER-ACE, gives up
$1.5$ points to DER++ with five epochs per task and takes $1.6$ back in a single pass. We
report both because neither protocol is neutral to all methods, and we did not re-tune the
local learner itself at batch $256$.
A second caveat concerns what the guarantee covers. It holds for the hidden computation of
the current batch, not for the output: the readout stays plastic by design, so a waking
prediction still changes on $0.33\%$ of samples per update, and on $0.001\%$ once the
readout is isolated as well. Final accuracy is not the quantity that invariance buys, which
is why Section~\ref{sec:online} measures the served prediction directly; what that section
cannot settle is how much a $0.33\%$ churn rate matters to a downstream controller, since
nothing here acts on its own predictions.
DER++ ties the system on split-MNIST and trails it only in the
single pass; on raw CIFAR-10 ER-ACE and DER++ lead it by two to three points, and BP+ER
keeps a $0.5$-point lead on feature-based CIFAR with both sides at their best, a lead the
substrate decomposition attributes to dense activation rather than to backpropagation
itself. The comparison with interleaved replay is therefore not a claim of a new accuracy
level; the contribution is consolidation with the live computation held invariant, at no
loss on split-MNIST against the strongest such reference and with the largest margins at
small buffers and in a single pass. On the backprop learner the isolation guarantee is free
only together with the rotation; isolation alone costs seven points there
(Section~\ref{sec:bpk}). Rotation keeps a static price of $1.3$ points on the local learner
and $0.5$ to $0.9$ on the backpropagation one at the $10\%$ fraction, which no gate we
tested removes without a sequential cost (Section~\ref{sec:ablation}). The system also makes
$35\times$ the night's updates at $6\times$ its wall-clock, in an implementation that
materialises a mask which is in fact an outer product of two indicator vectors, so the
figure bounds the idea's cost loosely from above. The same-substrate
offline night is highly variable across seeds ($89.0\pm3.5$ on the reported split,
$91.5\pm0.5$ on the second, where it trails the full system by only $0.3$), which we report
but do not explain. Finally, the novelty claims rest on our own literature search (August
2026); the individual primitives (refractory competition, sparse or null-space isolation,
replay, adaptive scheduling) all have ancestors, and the claimed contribution is their
combination into concurrent, exactly isolated consolidation.

Code for the experiments is at \url{https://github.com/Hanny658/IaC-Replay-in-SDF}.

\section{Future Work}
\label{sec:future}

(i) The bimodal instability of the same-substrate offline night at the $10\%$ fraction
($89.0\pm3.5$), and what makes offline consolidation fragile on wide substrates. (ii) The
bout gate's residual static price ($94.2$ vs.\ $96.0$ rotation-free, development-phase $5\%$
numbers): the OFF-side refractory, utility exemption and synaptic anchoring are all ruled
out as routes past it (Section~\ref{sec:ablation}), leaving mastery-annealed rotation depth
untried. (iii) The narrow-vs-wide substrate interaction of offline nights at tiny buffers.
(iv) Convolutional stacks. Depth and residual-style routing are answered
(Section~\ref{sec:controller}: five hidden layers at the two-layer level with locally
learned skips and the controller; on 100-class streams width, not depth, is what pays), but
convolution and a learned front-end are not, and the static-axis wall of the 100-class
regime ($6$--$8\%$ under every substrate change) indicates that the front-end is the
binding constraint there. (v) A throughput theory of the isolated channel that predicts the
required replay cadence from $\rho_\ell$, code overlap and the interference rate. (vi) The
interaction between rotation and spike calibration, which was optimiser-dependent in
development and small on the held-out split. (vii) A drive estimate that survives amplified
update streams (inflated targets, night gains), which currently defeat the controller's
clip.

\section{Conclusion}

In the evaluated class-incremental MLP settings, consolidation does not require a separate
offline phase. The invariance the isolation provides does not protect a finished forward
pass; it closes the channel through which consolidation would perturb the units the present
stream is using. Its accuracy value is $2.6$ points without rotation and $0.1$--$0.3$ with
it (Table~\ref{tab:ablation}). Once the coalition alternates, the live units tolerate
unmasked replay down to four-sample micro-batches; what isolation adds is the guarantee that
they need not, and, at two-sample batches where that tolerance runs out, graceful
degradation where the unmasked control failed severely in half the seeds. Measured while
the stream runs rather than at its end, that guarantee has a number of its own: the replay
update moves the prediction on the batch being served an order of magnitude less often than
unmasked replay does, while the rotation is what holds the retention curve flat and the
offline night, which repairs the old tasks only when it arrives, serves the stream ten
points lower than either (Table~\ref{tab:online}). With isolated
replay, refractory rotation of the active coalition, homeostatic burst triggering and a
relative-novelty gate on the rotation itself, a biologically constrained learner
consolidates between the steps of its stream. It is competitive with the tested
offline-night schedules on both held-out splits (above the best night by $2.6$ on the first
and by $0.3$ on the second) and consistently above ER, ER-ACE, A-GEM
and unmasked local replay, at $2.2\times$ the night's replay samples (and, on the first
split, at $1.1\times$); it leads every reference at a buffer of $200$ samples and in a
single pass over the stream, trails DER++ by $1.5$ points when the data is seen five times
and leads it by $1.6$ when the data is seen once, pays one to two points on i.i.d.\ streams, and runs at a
twentieth of dense-rate inference energy when spiking. The mechanism is not tied to the
local rule: on a backprop learner with the same sparse layers the rotation transfers, most
of all in a single pass, and the isolation guarantee again costs nothing on top of it. Where
it falls short, at large
buffers and on raw CIFAR-10 against the cross-entropy backpropagation references, the
shortfall is the local learner's, not the mechanism's: the offline night on the same
learner is further behind still, and removing the offline phase costs nothing anywhere we
measured.

\clearpage
\section*{Acknowledgments and Disclosure of Funding}
This research is supported by the A\raisebox{0.12ex}{$\star$}STAR-NTU Research Joint Lab for Smart and Sustainable Advanced Manufacturing, in part by MTI under its AI Centre of Excellence for Manufacturing (AIMfg) (Award W24MFMF001).

\clearpage
\appendix

\section{One training step, the three channels and the skip path}
\label{app:alg}
Algorithm~\ref{alg:step} lists one waking step of the default system with the state each
quantity is read from. Two conventions matter for the guarantees. First, the awake sets are
read from the waking forward pass, i.e.\ from the weights before the unmasked waking update,
while the masked replay step acts on the updated weights;
Propositions~\ref{prop:pre}--\ref{prop:post} hold verbatim for a mask computed on the weights
being updated, and the drift reported in Section~\ref{sec:isolation} is measured under the
implemented order. Second, the replay micro-batch is inferred with the suppression removed,
so a refractory unit may fire on replay and receive the update (Corollary~\ref{cor:refr}
makes that update invisible to the waking batch for any magnitude), and a burst of $n_b$
micro-batches reuses the mask of its waking batch. The waking step is $\eta_w=\eta\,m/256$
for a waking batch of $m$ samples (the per-epoch drift of a batch-$256$ protocol), the replay
step $\eta_r=3\eta$ per micro-batch, with $\eta=0.02$ throughout.

\begin{algorithm}[h]
\caption{One waking step of the default system (rotation always on, one burst per batch).}
\label{alg:step}
\begin{algorithmic}[1]
\REQUIRE weights $W_t$, velocities $v_t$, suppression $s(t)$, buffer $\mathcal{M}$, pressures $S$
\STATE $a\leftarrow\mathrm{forward}(X_t;\,W_t,\,s(t))$,\ \ $\varepsilon\leftarrow\mathrm{sweep}(a,y_t)$ \hfill Eqs.~\eqref{eq:forward}--\eqref{eq:sweep}
\STATE $(W,v)\leftarrow$ unmasked heavy-ball step on $(W_t,v_t)$ with $G^\ell=\varepsilon^\ell(a^{\ell-1})^{\!\top}$, rate $\eta_w$ \hfill wake update
\STATE $\mathcal{A}^\ell\leftarrow\{i:\exists b,\ a^\ell_{i,b}\neq0\}$ for every hidden layer;\ \ $S_i\leftarrow S_i+\bar a_i(X_t)$ \hfill from the pre-update $a$
\STATE $M^\ell\leftarrow$ Eq.~\eqref{eq:mask}, unit mask $\mathbf{1}[i\notin\mathcal{A}^\ell]$, readout row and bias unmasked
\STATE $s^\ell_i(t{+}1)\leftarrow\mathbf{1}[i\in\mathcal{A}^\ell]$ \hfill refractory rotation
\STATE write $(X_t,y_t)$ into $\mathcal{M}$ (reservoir)
\FOR{each of the $n_b$ replay micro-batches $(X_r,y_r)\sim\mathcal{M}$}
\STATE $a_r\leftarrow\mathrm{forward}(X_r;\,W,\,s{=}0)$,\ \ $\varepsilon_r\leftarrow\mathrm{sweep}(a_r,y_r)$ \hfill suppression removed
\STATE $(W,v)\leftarrow$ masked step, Eq.~\eqref{eq:maskedstep}, with $G^\ell=\varepsilon^\ell_r(a^{\ell-1}_r)^{\!\top}$, rate $\eta_r$, masks $M^\ell$
\ENDFOR
\STATE $S_i\leftarrow0$ for every $i\notin\mathcal{A}^\ell$ if a burst ran \hfill discharge
\RETURN $W_{t+1}=W$,\ $v_{t+1}=v$,\ $s(t{+}1)$
\end{algorithmic}
\end{algorithm}

\paragraph{The three channels.} Table~\ref{tab:channels} lists the three ways a synapse
enters the mask \eqref{eq:mask} and what each guarantees; the rotation of
Section~\ref{sec:rotation} exists to move synapses from the second row into the third.

\begin{table}[h]
\centering\small
\begin{tabular}{llll}
\toprule
Channel & Condition & Invisible to $X$ & Source \\
\midrule
presynaptic silent & $j\notin\mathcal{A}^{\ell-1}(X)$ & for any magnitude & Prop.~\ref{prop:pre} \\
postsynaptic naturally asleep, active pre & $i\notin\mathcal{A}^{\ell}(X)$, $s^\ell_i=0$ & under the margin $\tau^\ell_{k,b}$ & Prop.~\ref{prop:post} \\
postsynaptic suppressed & $s^\ell_i=1$ & for any magnitude & Cor.~\ref{cor:refr} \\
\bottomrule
\end{tabular}
\caption{\textbf{The three consolidation channels} opened by the synapse mask.}
\label{tab:channels}
\end{table}

\paragraph{A numerical example.} On one sample with $k=2$ and $\sigma(z)=(3,2,0.5,0)$, units
$3$ and $4$ are asleep and $\tau_k=2$. Any replay update with $\sigma(z_3+\delta_3)<2$ and
$\sigma(z_4+\delta_4)<2$ leaves the activity vector $(3,2,0,0)$ unchanged, because the two
winners' own inputs are not touched; if unit $4$ is refractory ($s_4=1$) its bound is void and
$\delta_4$ may be arbitrarily large; and every synapse leaving unit $4$ into the next layer may
move by any amount, since its presynaptic activity is zero.

\paragraph{Skip synapses.} For $\ell\ge3$ each deep hidden layer receives a bypass,
$z^\ell=W^\ell a^{\ell-1}+S^\ell a^{\ell-2}+b^\ell$, with $S^\ell$ under the same Dale
projection and a feedback twin $B_S^{\ell}\in\mathbb{R}^{n_{\ell+2}\times n_\ell}$ that carries the
bypass target's burst back to its source,
\[
\varepsilon^{\ell}\leftarrow\varepsilon^{\ell}+\sigma'(z^{\ell})\odot\Big((B_S^{\ell})^{\!\top}\big[\kappa\tanh(\varepsilon^{\ell+2}/\kappa)\odot\mathbf{1}(a^{\ell+2}>0)\big]\Big),
\]
$S^\ell$ and $B_S^{\ell-2}$ learned by the local product of \eqref{eq:kp} with the target
layer's decay $\lambda_\ell$, and replay-masked by the same rule two layers apart,
$M^{\ell}_{S,ij}=\mathbf{1}[\,j\notin\mathcal{A}^{\ell-2}(X)\ \lor\ i\notin\mathcal{A}^{\ell}(X)\,]$,
so that Propositions~\ref{prop:pre}--\ref{prop:post} extend unchanged.

\section{The price of the substrate}
\label{app:substrate}

Table~\ref{tab:substrate} traces the path from a dense backpropagation MLP to the substrate
used throughout, on i.i.d.\ MNIST where each constraint's cost is cleanest. The upper block is
cumulative; the lower block replaces one ingredient at its stage and is to be compared with
the matching cumulative row.

\begin{table}[t]
\centering
\caption{\textbf{Substrate ablation} on i.i.d.\ MNIST ($256$--$128$ hidden core, three seeds;
development phase, official test set).
``Syn.\ ops'' is the fraction of the dense MLP's synaptic operations in one forward pass.
Upper block: constraints added cumulatively; lower block: one ingredient replaced at its stage.}
\label{tab:substrate}
\small
\begin{tabular}{@{}lcc@{}}
\toprule
Substrate & Accuracy & Syn.\ ops \\
\midrule
dense backpropagation MLP (autograd reference) & $97.7\pm0.1$ & 1.00 \\
local burst-error learner, unconstrained & $97.8\pm0.1$ & 0.85 \\
$+$ bounded error, Dale's law, sign-concordant feedback (no transport) & $97.1\pm0.0$ & 0.85 \\
$+$ \kwta{} $10\%$ & $97.0\pm0.2$ & 0.80 \\
$+$ relaxation $T{=}5$ & $97.0\pm0.1$ & 0.80 \\
$+$ $30\%$ distance-dependent connectivity & $96.5\pm0.2$ & 0.24 \\
$+$ burst gate and burst baseline & $96.3\pm0.2$ & 0.24 \\
$+$ single-phase burst sweep replacing relaxation (\emph{final substrate}) & $96.5\pm0.1$ & 0.24 \\
\midrule
random wiring instead of distance-dependent, equal density & $95.4\pm0.0$ & 0.24 \\
one settling step ($T{=}1$) instead of five & $17.9\pm6.8$ & 0.80 \\
weight transport restored (mirroring, no KP decay) & $96.3\pm0.2$ & 0.24 \\
Kolen--Pollack decay removed & $96.3\pm0.1$ & 0.24 \\
heavy-ball SGD instead of Adam ($\eta=0.02$) & $95.2\pm0.3$ & 0.24 \\
\kwta{} $5\%$ / $2\%$ instead of $10\%$ & $95.9\pm0.1$ / $94.2\pm0.2$ & 0.79/0.78 \\
multiplicative burst coding & $78.8\pm1.9$ & 0.24 \\
training on $T_s{=}8$ spike counts & $80.0\pm0.8$ & 0.24 \\
\bottomrule
\end{tabular}
\end{table}

The cumulative ladder prices the entire biological constraint set at $1.2$ points below the
dense backpropagation reference ($97.7\to96.5$) for $4\times$ fewer synaptic operations per
pass; the local burst-error objective itself costs nothing ($97.8$ unconstrained). No single
constraint is expensive: the bounded-error/Dale/no-transport trio costs $0.7$, $10\%$
\kwta{} one tenth of a point, and the $70\%$ connectivity cut half a point, of which
distance-dependent wiring recovers $+1.1$ over random wiring at equal density. The lower
block adds three cautions. Settling cannot merely be shortened ($T{=}1$ collapses to
$17.9$); the single-phase burst sweep is the principled replacement, matching the $T{=}5$
accuracy in a single pass. The Kolen--Pollack decay and even fully restored weight transport
are accuracy-neutral at this scale; the decay is useful as a regulariser on harder inputs,
not as an aligner. And the two spiking-flavoured training schemes fail outright: spiking is
affordable at \emph{inference} (the energy measurements above), not as a training-time code.

Finally, the two tables disagree about the optimiser in an instructive way. On the plain
substrate SGD \emph{loses} $1.3$ points to Adam ($95.2$ vs.\ $96.5$), yet inside the
isolated-consolidation loop the two tie (Table~\ref{tab:ablation}). Per-synapse adaptive
state is an asset while every synapse trains on every batch, and no longer one once updates
become intermittently mask-confined. The optimiser is thus not a neutral implementation
detail: its advantage on the static substrate does not carry into the consolidating system,
where the least state suffices.

\section{The per-batch rotation gates under masked Adam}
\label{app:gates}
Appendices~\ref{app:gates}--\ref{app:probes} report development-phase numbers (official
test set) from the earlier $5\%$ phases, the small-buffer and down-selection studies and the
learning-rate sweeps; they were not re-run under the held-out protocol and support mechanism orderings only, except
where a paragraph is marked held-out (the single-pass streams of Appendix~\ref{app:probes}).

Figure~\ref{fig:rotation}: always-on rotation costs $1.3$--$3.0$ points on i.i.d.\ MNIST
($95.0\to92.0$--$93.7$). Gating rotation by \emph{unit-level pressure} only trades the axes
(static $91.6\to94.0$ as sequential $88.8\to80.7$; no dominant point). The \emph{relative
novelty} gate \eqref{eq:nov} removes the price at no sequential cost on MNIST: static
$\mathbf{95.0\pm0.4}$ (rotation on $1.4$--$17\%$ of batches), sequential $90.3\pm0.4$
($\beta{=}1.5$) or $90.9\pm0.6$ ($\beta{=}1.1$). On split CIFAR-10 the price \emph{reverses}:
rotation gains $+2.5$ (refractory) to $+6.5$ points (pressure-selected rest) on the static
axis, because forced rotation acts as a regulariser in the underfit regime, while the pure
novelty gate \eqref{eq:nov} \emph{misfires} there (sequential $19.1\pm1.1$: fast/slow $\to1$
on a stream that never becomes predictable, so rotation shuts off exactly where it helps).
The gate with the unmastered clause \eqref{eq:prog} resolves this with one setting across
all four regimes: split CIFAR-10 sequential $27.2\pm0.7$ and static $25.3\pm1.0$ (both $=$
always-on rotation), split-MNIST $90.4\pm0.9$, i.i.d.\ MNIST $94.3\pm0.2$ at $\gamma{=}0.5$.
It retains a residual $0.7$-point static cost relative to the pure novelty gate on easy
streams, the price of the unmastered clause. All gate numbers in this appendix use masked
Adam; under heavy-ball SGD the gate must be committed to minimum-duration bouts
(Section~\ref{sec:ablation}).

\begin{figure}[t]
\centering
\includegraphics[width=\linewidth]{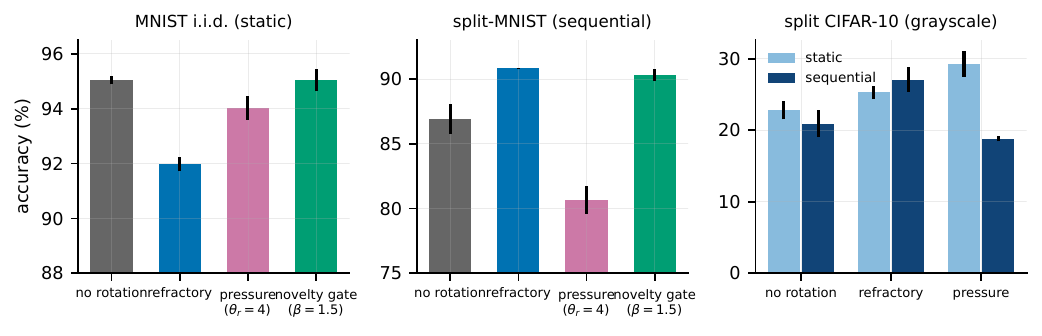}
\caption{\textbf{Rotation: price, trade-off, and gate.} Left: on i.i.d.\ MNIST the novelty gate
restores no-rotation accuracy. Middle: on split-MNIST rotation is required and the gate keeps it.
Right: on split CIFAR-10 (grayscale) rotation is a gain on \emph{both} axes.}
\label{fig:rotation}
\end{figure}

\section{Small buffers: what the night still buys}
\label{app:k200}

At $K{=}200$ the night at first appeared to keep a $\sim$2.6-point edge; most of it turns
out to be a substrate mismatch (Figure~\ref{fig:k200}). On the \emph{same}
$512$--$256$/$5\%$ substrate, the night reaches $84.0\pm3.4$ while noisy pressure-burst
local sleep reaches $84.3\pm1.0$ and local$+$night $84.8\pm0.9$: parity, with three-fold
smaller variance for the wake-side scheme. The strongest tiny-buffer number remains the
narrower $256$--$128$/$10\%$ night ($86.9$), so the offline window retains an advantage only
in combination with narrow substrates. Two mechanism results hold regardless. Replay
diversity (noise $\sigma{=}0.2$) is specifically a wake-side lever: it lifts every local
scheme ($81.8\pm2.9\to84.3\pm1.0$) and \emph{hurts} the night ($84.0\to81.8\pm3.2$). And
\emph{unmasked}, rotation-free waking replay collapses ($71.0\pm1.7$), so isolated rotating
replay beats it at every buffer size we tested. On the held-out split at the $10\%$
fraction (Section~\ref{sec:r1}, six seeds) the picture at $K{=}200$ is parity with the
narrow night ($83.0\pm1.4$ against $82.7\pm2.9$) and a clear lead over DER++ ($77.1$) and
BP+ER ($68.5$), so the night's small-buffer edge in this appendix belongs to the $5\%$
development phase.

\begin{figure}[t]
\centering
\includegraphics[width=.6\linewidth]{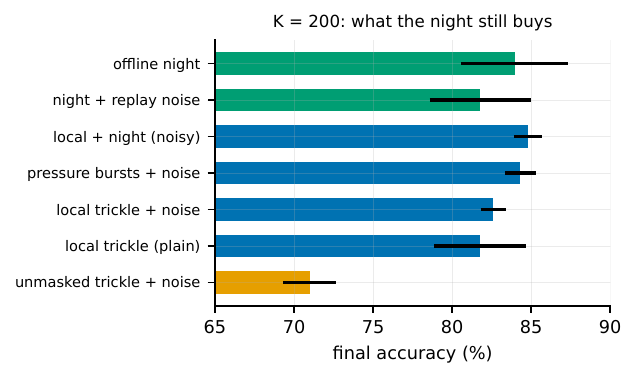}
\caption{\textbf{$K=200$ mechanism triangulation.} Isolated rotating replay beats the unmasked trickle, diversity helps only during wake, and the irreducible advantage of the night at tiny buffers is its interference-free window.}
\label{fig:k200}
\end{figure}

\section{Sleep-anchored down-selection}
\label{app:downsel}

Motivated by the synaptic homeostasis hypothesis \citep{tononi2014}, we anchored structural
plasticity (prune the weakest fraction of live synapses, regrow the same number with a distance
bias) to consolidation windows, using the Kolen--Pollack decay of \eqref{eq:kp} as the natural
weakening force on unreplayed synapses. The outcome splits three ways at matched turnover
($\approx5\%$ per epoch). On the \emph{continuous} local-sleep substrate pruning is free
(sleep-anchored $90.9\pm0.2$, clocked $90.6\pm0.4$, none $90.8\pm0.1$). On the pure
\emph{night} substrate, pruning immediately after each night helps substantially
($87.2\pm3.8$ vs.\ $84.1\pm1.3$): offline down-selection rescues the wide-sparse night.
On the \emph{episodic-burst} substrate pruning is harmful ($84.6\pm0.5$ burst-anchored,
$86.6\pm1.3$ clocked, vs.\ $89.9\pm0.1$ without): when replay is rare, synapses are cut before
enough consolidation events have protected them. Down-selection, like replay itself, needs
either density or an offline window.

\section{Learning-rate sensitivity}
\label{app:eta}

To ensure that the optimiser comparison is not an artefact of tuning SGD's $\eta$ while
leaving Adam's at its default, both were swept. At the $10\%$ fraction (three seeds),
heavy-ball SGD under always-on rotation holds $92.2\pm0.7/93.3\pm0.2$ at $\eta{=}0.01$,
$92.7\pm0.3/94.7\pm0.3$ at $0.02$ (the joint optimum) and $91.5\pm0.6/94.7\pm0.3$ at $0.05$,
decaying at $0.005$ ($91.3/91.3$) and destabilising at $0.1$ ($79.5\pm17.3$), while masked
Adam has its optimum at the default ($91.8\pm0.6/94.6\pm0.2$ at $10^{-3}$; $89.9/89.9$ at
$3{\cdot}10^{-4}$, $90.9/94.3$ at $3{\cdot}10^{-3}$). The comparison of
Table~\ref{tab:ablation} is therefore at each optimiser's best rate. The sweep at the $5\%$
fraction tells the same story. There, under always-on rotation, SGD holds a sequential
plateau over $\eta\in[0.01,0.02]$ ($91.9\pm0.5$/$92.0\pm0.6$) and a static plateau over
$[0.02,0.05]$ ($93.3$--$93.7$), decaying gracefully outside ($86.3$/$89.2$ at $\eta{=}0.1$),
with $\eta{=}0.02$ the joint optimum. Masked Adam, swept over a tenfold range, has its
sequential optimum exactly at the default ($90.8\pm0.1$ at $10^{-3}$; $88.6$ at
$3{\cdot}10^{-4}$, $88.9$ at $3{\cdot}10^{-3}$), and no Adam setting reaches the heavy-ball
joint point ($3{\cdot}10^{-3}$ buys $93.6$ static at a three-point sequential cost).
Unmasked SGD is poor and unstable at \emph{every} rate ($84$--$88$, seed deviations
$1.8$--$4.1$, three- to eight-fold those of the masked runs), so the stabilising role of
isolation plus rotation (the unmasked runs are rotation-free) holds across $\eta$. The
bout-committed gate inherits the plateau, and the optimum shifts one notch down on CIFAR
($\eta{=}0.01$).

\section{Probes of the residual static price}
\label{app:probes}

\paragraph{Single-pass streams.} With one epoch per task instead of five (everything else
as in the headline cells, six seeds, held-out split), isolated replay with rotation reaches
$91.8\pm0.3$ ($F{=}5.5$) against $91.6$ with five epochs, DER++ $90.1\pm0.7$, BP+ER
$88.6\pm0.2$, ER-ACE $88.5\pm0.3$, the offline night $76.9\pm3.4$ (one night per task is too
few), unmasked interleaved replay $83.0\pm14.9$ and no buffer $18.1$: the proposed system
keeps its lead over every tested control in a single pass, while the offline night falls
below BP+ER. The single-pass column of Table~\ref{tab:ablation} gives every ablation row
under this protocol; on the second split the full system reaches $91.8\pm0.1$ and DER++
$89.4\pm0.3$.

\paragraph{Bout length and the OFF-side refractory.}
Bout length is a genuine control parameter: $M{=}512$ (duties $0.61/0.38$) gives
$91.7\pm0.4/94.0\pm0.5$, already recovering most of what the per-batch gate loses, while
$M{=}4096$ drives the static duty to $0.94$ and the refund vanishes ($93.4\pm0.1$). The
static ceiling of silent isolation ($96.0$) stays out of reach, because rotation keeps a
residual price even in bouts; the committed gate buys regime-adaptivity, not a static gain.
Completing the flip-flop with an OFF-side refractory (once a bout expires, triggers are
ignored for $M_{\mathrm{off}}$ batches) is a clean null result. Short refractories
($M_{\mathrm{off}}{\le}1024$) never engage on the static stream, since a bout expires only
after its trigger has been quiet for $M$ batches and static re-triggering is therefore not
flicker to begin with; a long one ($M_{\mathrm{off}}{=}4096$) lowers the static duty to
$0.53$ without moving static accuracy ($94.0\pm0.7$) and monotonically erodes the
sequential axis ($92.1\to91.0$, occasionally masking a task switch). Together with the
$M{=}512$ cell (duty $0.38$, static $94.0$), this locates the residual static price in the
rotation of settled coalitions itself, not in how often the gate fires.

\paragraph{Structural exemption and synaptic anchoring.}
\emph{Utility-exempt rotation} spares the top-$q$ units by long-term use trace from ever
being benched, so the settled coalitions stay awake. It traces a smooth trade-off between
always-on rotation and the rotation-free system ($q{=}0.10$: $90.9\pm0.1/94.2\pm0.2$;
$q{=}0.25$: $89.1\pm0.4/95.4\pm0.2$), but the bout gate beats the $q{=}0.10$ point by $+1.1$
sequential at equal static. \emph{Temporal commitment dominates structural exemption},
which repeats at the utility level what unit-level pressure gating showed earlier
(Figure~\ref{fig:rotation}): any structural narrowing of the rotation pool is paid for in
replay-channel width.

\emph{Two-timescale synapses} give every weight a slow anchor in the spirit of synaptic
consolidation cascades \citep{benna2016}: the fast weight decays toward its anchor at rate
$\lambda$ while the anchor absorbs it at rate $\mu$, ticking only on waking updates so the
exactness of the replay path is untouched. The strong hypothesis, that settled function
living in the anchor would dissolve rotation's static price, is falsified: across
$\lambda\in[3{\cdot}10^{-4},3{\cdot}10^{-3}]$, $\mu\in[10^{-4},10^{-3}]$ no cell exceeds
the anchor-free static accuracy, and strong coupling hurts both axes (the anchor becomes a
lagging drag). Weak symmetric coupling ($\lambda{=}\mu{=}3{\cdot}10^{-4}$), however, acts as
a \emph{sequential stabiliser}: $92.4\pm0.2$, the best sequential cell of the $5\%$ grid,
with the seed variance halved, at static parity, weakly dominating the $5\%$ always-on
point. Anchor and bout gate compose sub-additively
($92.3\pm0.3/93.8\pm0.5$, between the parents).


\section{Isolation across replay batch sizes}
\label{app:isomargin}
Figure~\ref{fig:isomargin} extends the left panel of Figure~\ref{fig:batch} to two-sample
replay micro-batches and shows, per batch size, the seed-paired margin of the full system over
rotation with unmasked replay (held-out split; six seeds at batches $2$, $4$ and $16$, three at
$8$). From $16$ down to $4$ the margin is $0.1$ $[-0.3,+0.5]$, $0.3$ $[-0.2,+0.6]$ and $0.3$
$[-0.5,+1.2]$: isolation costs nothing and buys nothing measurable in mean accuracy while the
rotation keeps the live coalition tolerant of unmasked replay. At batch $2$ the replay
gradient is too noisy for either system. The isolated one falls to $79.0\pm3.2$ (forgetting
$21.5\pm4.2$ against $6.4$ at batch $16$), with every seed between $75.5$ and $84.3$; the
unmasked one bifurcates, with three seeds finishing at $82.5$--$84.8$, slightly \emph{above}
their isolated counterparts ($+0.3$, $+1.8$, $+8.1$), one at $58.2$ and two at chance
($9.9$). Leaking replay onto the live coalition is therefore not a pure loss, since the
surviving seeds show that the extra plasticity can help, but it exposes a severe-failure
mode (final accuracy below $60\%$: two seeds at chance, one at $58.2$) that the isolated
system did not show in six seeds. With the current input's hidden computation held
invariant, the noisiest replay degraded the system gracefully in every seed tested. Six
seeds without a severe failure bound its frequency but do not prove the mode absent, and the
masked system itself destabilises at replay gain $10$ (Section~\ref{sec:ablation}). The
paired mean at batch $2$ ($+24$ $[-0.2,+48.5]$) describes the unmasked failures, not a
margin, and is kept out of the ranges quoted in the main text.

\begin{figure}[h]
\centering
\includegraphics[width=.95\linewidth]{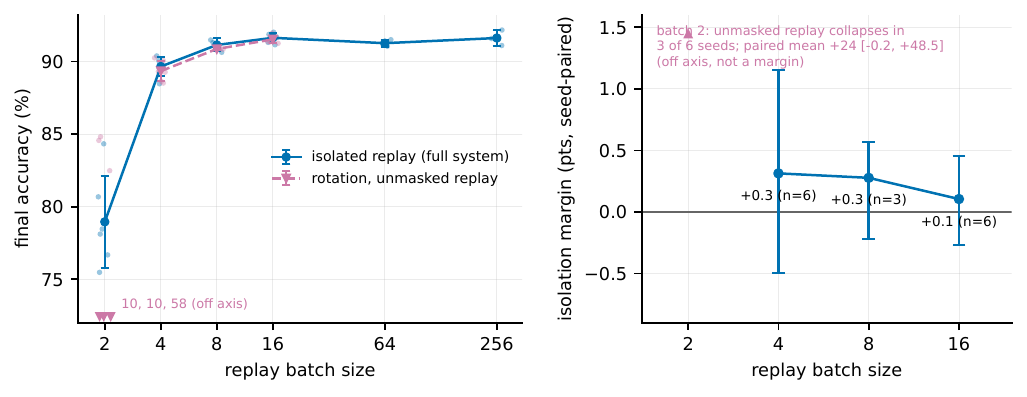}
\caption{\textbf{Isolation across replay batch sizes} (held-out split). \textbf{Left:}
split-MNIST accuracy against replay batch size for isolated replay (the full system) and for
rotation with unmasked replay, every seed shown; at batch $2$ no unmasked mean is drawn
because three of six seeds fail severely (values at the axis floor). \textbf{Right:} the seed-paired
margin of isolation with a 95\% bootstrap interval at the batch sizes where the unmasked
run does not collapse.}
\label{fig:isomargin}
\end{figure}

\section{Development-phase numbers and a second held-out split}
\label{app:val}

Configurations were selected on the official test splits (the development set) and frozen;
the comparisons in the main text re-train them on nine tenths of the training data and
evaluate on the held-out tenth, which took part in no selection. Table~\ref{tab:val} places
the development-set numbers beside the held-out ones and beside a second held-out tenth for
the headline rows (drawn with a different seed; the two tenths share $9\%$ of their samples,
so they are not disjoint test sets). Over the $87$ sequential configurations run under both protocols, reported
numbers sit $1.1$ points below the development numbers on average (a smaller training set
and a distribution-matched evaluation set both contribute) and rank-correlate with them at Spearman $\rho=0.97$ (a robustness check, not a proof that selection
bias is absent); every ordering carried by the paper holds on both held-out splits, with
one magnitude that does not: the same-substrate night, bimodal on the first split
($89.0\pm3.5$), reaches $91.5\pm0.5$ on the second and trails the full system by $0.3$ there.
The adaptive decay's static margin on MNIST is likewise split-dependent ($93.5\pm4.3$ with one
collapsed seed on the first split, $95.2\pm0.4$ on the second). Table~\ref{tab:accmatrix}
gives the per-task accuracy matrix of the headline configuration on the reported split.

\begin{table}[H]
\centering\small
\caption{\textbf{Development-set (official test split) numbers beside the held-out numbers and
a second held-out split.} Sequential accuracy, with the static axis on its own row
where the paper reports it; six seeds on the reported split, three on the development set
and on the second split. Lower blocks: single pass, and raw split CIFAR-10.}
\label{tab:val}
\begin{tabular}{@{}lccc@{}}
\toprule
& development set & held-out split 1 & held-out split 2 \\
\midrule
always-on rotation (default) & $92.7\pm0.3$ & $91.6\pm0.3$ & $91.8\pm0.4$ \\
\quad static & $94.7\pm0.3$ & $94.1\pm0.2$ & $94.1\pm0.6$ \\
bout-committed gate & $92.3\pm0.2$ & $91.8\pm0.3$ & $92.0\pm0.6$ \\
rotation alone (unmasked replay) & $92.1\pm0.5$ & $91.5\pm0.3$ & $91.8\pm0.6$ \\
\quad static & $94.7\pm0.1$ & $93.7\pm0.3$ & $93.7\pm0.3$ \\
silent mask (no rotation) & $88.8\pm0.8$ & $88.0\pm0.7$ & $88.6\pm1.3$ \\
unmasked interleaved ER (SGD) & $87.0\pm1.1$ & $85.4\pm3.5$ & $80.0\pm14.1$ \\
BP $+$ ER & $89.3\pm0.8$ & $88.8\pm0.3$ & $88.5\pm0.7$ \\
offline night, same substrate & $90.9\pm3.2$ & $89.0\pm3.5$ & $91.5\pm0.5$ \\
offline night, narrow substrate & $90.7\pm0.4$ & $88.5\pm2.8$ & $89.3\pm0.6$ \\
adaptive $\lambda$ (drive) & $91.7\pm0.2$ & $90.6\pm0.6$ & $90.8\pm0.1$ \\
\quad static & $96.1\pm0.2$ & $93.5\pm4.3$ & $95.2\pm0.4$ \\
five hidden layers $+$ skips $+$ adaptive $\lambda$ & $91.5\pm0.4$ & $91.0\pm0.3$ & $91.1\pm0.3$ \\
\quad static & $96.2\pm0.3$ & $95.8\pm0.3$ & $95.7\pm0.1$ \\
no buffer & $19.6\pm0.0$ & $19.4\pm0.0$ & $19.5\pm0.0$ \\
random rotation & $90.8\pm1.3$ & $91.1\pm0.7$ & $90.7\pm0.5$ \\
replay confined to the readout & $27.3\pm15.1$ & $41.7\pm10.9$ & $38.1\pm4.9$ \\
mirror isolation & $77.9\pm1.6$ & $79.4\pm3.5$ & $79.2\pm2.4$ \\
soft rotation & $90.2\pm0.4$ & $89.9\pm0.7$ & $89.5\pm0.4$ \\
BP $+$ DER++ & $92.2\pm0.5$ & $91.5\pm0.5$ & $91.3\pm0.3$ \\
BP $+$ ER-ACE & $90.4\pm0.3$ & $90.2\pm0.4$ & $90.4\pm0.2$ \\
BP $+$ A-GEM & $73.6\pm4.2$ & $68.5\pm7.0$ & $70.4\pm4.1$ \\
BP $+$ DER++, waking batch $16$ & $93.4\pm0.1$ & $93.1\pm0.6$ & $93.2\pm0.6$ \\
BP $+$ ER-ACE, waking batch $16$ & $90.3\pm0.5$ & $89.7\pm0.4$ & $89.9\pm0.3$ \\
BP $+$ ER, waking batch $16$ & $90.9\pm0.3$ & $90.8\pm0.6$ & $90.3\pm0.8$ \\
BP $+$ \kwta{}, interleaved replay (same schedule) & $92.4\pm0.4$ & $91.7\pm0.4$ & --- \\
BP $+$ \kwta{}, isolated replay & $84.3\pm2.5$ & $84.7\pm2.6$ & --- \\
BP $+$ \kwta{}, interleaved replay $+$ rotation & $92.9\pm0.4$ & $92.2\pm0.3$ & --- \\
BP $+$ \kwta{}, isolated replay $+$ rotation & $92.9\pm0.2$ & $92.1\pm0.3$ & --- \\
\midrule
\multicolumn{4}{@{}l}{\emph{single pass (one epoch per task)}} \\
\quad isolated replay $+$ rotation & $91.8\pm0.4$ & $91.8\pm0.3$ & $91.8\pm0.1$ \\
\quad unmasked interleaved replay & $68.3\pm19.1$ & $83.0\pm14.9$ & $83.9\pm6.0$ \\
\quad offline night, same substrate & $75.4\pm0.5$ & $76.9\pm3.4$ & $75.2\pm5.5$ \\
\quad BP $+$ ER & $90.1\pm0.7$ & $88.6\pm0.2$ & $88.2\pm0.9$ \\
\quad BP $+$ DER++ & --- & $90.1\pm0.7$ & $89.4\pm0.3$ \\
\midrule
\multicolumn{4}{@{}l}{\emph{split CIFAR-10 (raw)}} \\
\quad isolated replay $+$ rotation & $29.5\pm0.8$ & $28.4\pm0.8$ & $29.7\pm0.9$ \\
\quad offline night & $26.2\pm0.6$ & $25.1\pm2.7$ & $26.5\pm0.4$ \\
\quad BP $+$ ER & $25.1\pm1.5$ & $24.6\pm0.6$ & $25.8\pm0.8$ \\
\quad BP $+$ DER++ & $31.1\pm0.3$ & $30.7\pm0.8$ & $31.0\pm0.6$ \\
\quad BP $+$ ER-ACE & $31.7\pm0.8$ & $31.6\pm1.0$ & $31.7\pm0.8$ \\
\quad unmasked ER & $17.0\pm6.1$ & $14.6\pm4.6$ & $15.5\pm4.5$ \\
\quad no buffer & $16.4\pm0.2$ & $16.3\pm0.1$ & $16.6\pm0.0$ \\
\bottomrule
\end{tabular}
\end{table}

\begin{table}[H]
\centering\small
\caption{\textbf{Per-task accuracy matrix} of the headline configuration on split-MNIST
(held-out split, mean over six seeds): row $=$ after training task $i$, column $=$ accuracy on
task $j$. Forgetting $F$, the mean of the diagonal minus the last row over tasks $1$--$4$, is
$6.4\pm0.6$; the drop is concentrated on tasks $2$--$3$ and the first task is retained.}
\label{tab:accmatrix}
\begin{tabular}{lccccc}
\toprule
after task & T1 & T2 & T3 & T4 & T5 \\
\midrule
1 & $99.8$ & & & & \\
2 & $98.7$ & $97.2$ & & & \\
3 & $97.7$ & $92.7$ & $97.0$ & & \\
4 & $97.5$ & $89.6$ & $92.9$ & $97.5$ & \\
5 & $96.7$ & $88.1$ & $89.7$ & $91.5$ & $91.8$ \\
\bottomrule
\end{tabular}
\end{table}

\end{document}